\documentclass[12pt]{article}
\usepackage{ajl_functions}

\newtheorem{theorem}{Theorem}
\newtheorem{corollary}{Corollary}
\newtheorem{lemma}{Lemma}

\begin{document}


\title{Dimensionality Reduction for Binary Data through the Projection of Natural Parameters}
\author{Andrew J. Landgraf \textsc{and} Yoonkyung Lee\\{\small Department of Statistics, The Ohio State University}}
\date{}
\maketitle

\begin{abstract}
Principal component analysis (PCA) for binary data, known as logistic PCA, has become a popular alternative to dimensionality reduction of binary data. 
It is motivated as an extension of ordinary PCA by means of a matrix factorization, akin to the singular value decomposition, that maximizes the Bernoulli log-likelihood.
We propose a new formulation of logistic PCA which extends Pearson's formulation of a low dimensional data representation with minimum error to binary data. Our formulation does not require a matrix factorization, as previous methods do, but instead looks for projections of the natural parameters from the saturated model. Due to this difference, the number of parameters does not grow with the number of observations and the principal component scores on new data can be computed with simple matrix multiplication. We derive explicit solutions for data matrices of special structure and provide computationally efficient algorithms for solving for the principal component loadings. Through simulation experiments and an analysis of medical diagnoses data, we compare our formulation of logistic PCA to the previous formulation as well as ordinary PCA to demonstrate its benefits. 
\end{abstract}


\section{Introduction}
\label{sec:1}

Principal component analysis (PCA) is perhaps the most popular dimensionality reduction technique \citep[see][for example]{jolliffe2005principal}. It is useful for data compression, visualization, and feature discovery. 
PCA can be motivated either by maximizing the variance of linear combinations of the variables \citep{hotelling1933pca} or by minimizing the reconstruction error of a lower dimensional projection of the cases \citep{pearson1901pca}.
There is an implicit connection between standard PCA and the Gaussian distribution in Pearson's formulation. 
\cite{tipping1999probabilistic} also showed that PCA provides the maximum likelihood estimate for a factor model, where the data are assumed to be Gaussian. 

Although PCA  is commonly used for dimensionality
reduction for various types of data in practice, the fact that PCA finds a low-rank
subspace by implicitly minimizing the
reconstruction error under the squared error loss renders direct application
of PCA to non-Gaussian data such as binary responses or counts  conceptually unappealing.
Moreover, the probabilistic interpretation of PCA with normal likelihood in
\cite{tipping1999probabilistic} suggests the possibility of proper likelihood-based loss
functions in defining the best subspace of a given rank for other types
of data. With this motivation, 
\cite{CollinsEtAl2001} proposed a generalization of PCA to
exponential family data  using the generalized linear model (GLM) framework,
and \cite{ScheinEtAl2003}, \cite{Tipping1998} and
\cite{Deleeuw2006}  examined similar generalizations for binary data in
particular, using the Bernoulli likelihood, which is  referred to as logistic PCA.
Generalized PCA estimates the natural parameters of a data matrix in a lower dimensional subspace by minimizing the negative log-likelihood under an exponential family distribution. In the Gaussian case, generalized PCA is shown to be equivalent to the truncated singular value decomposition (SVD). 

In this paper, we argue that \cite{CollinsEtAl2001}'s logistic PCA is more closely related to SVD than PCA because it aims at a low-rank factorization of the natural parameters matrix. Consequently, each case has its own latent factor associated with it and the number of parameters increases with the number of cases. The drawback of the formulation becomes apparent when it comes to prediction. To apply logistic PCA to new data, one needs to carry out another matrix factorization, which is prone to overfit. 
This is in contrast with standard PCA where the principal component scores for the new data are simply given by linear combinations of the observed values of the variables.

Retaining the structure of standard PCA,
we generalize PCA in such a way that the principal component scores are linear functions of the data. This is done by interpreting \cite{pearson1901pca}'s formulation in a slightly different manner. A projection of the data with minimum reconstruction error under squared loss can be viewed alternatively as a projection of the {\it natural parameters} of a saturated model with minimum {\it deviance} for Gaussian data. This alternative interpretation allows a coherent generalization of standard PCA to exponential family distributions. When the distribution is Gaussian, this generalization simplifies to standard PCA.
Due to the prevalence of binary data and for simplicity of exposition of the new generalization of PCA, we focus on logistic PCA in this paper.

Our formulation has several benefits over \cite{CollinsEtAl2001}'s formulation. The number of parameters does not increase with the number of observations, the principal component scores are easily interpretable as linear functions of the data, and applying principal components to a new set of data only requires a matrix multiplication. Furthermore, while very little is known about solutions to \cite{CollinsEtAl2001}'s formulation, some explicit solutions to our formulation can be derived for data matrices with special structures.

Computationally, our formulation of logistic PCA, like the previous one, leads to a non-convex problem. We derive two practical algorithms for generating solutions. The first algorithm involves iterative spectral decompositions,
using the majorization-minimization (MM) algorithm \citep[see][]{hunter2004MM}.
The second algorithm involves a convex relaxation of low-rank projection matrices. Using these algorithms, we apply our formulation of logistic PCA to several datasets, examining the advantages and trade-offs with existing methods.

The rest of the paper is organized as follows. Section \ref{sec:background} gives background on PCA and logistic PCA. In Section \ref{sec:formulation}, we introduce our formulation of generalized PCA and qualitatively compare it to the previous formulation. In Section \ref{sec:patterned}, we derive the first-order optimality conditions for logistic PCA solutions and characterize explicit solutions that satisfy the conditions for data matrices with special structures. In Section \ref{sec:computation}, we derive two algorithms for logistic PCA. Section \ref{sec:examples} shows the potential benefits of our formulation via data analyses with simulated and real data. Finally, Section \ref{sec:discussion} concludes the paper with a discussion and possible extensions of the proposed logistic PCA to exponential family data.

\section{Background}
\label{sec:background}
\label{sec:2a}

\cite{pearson1901pca} considered a geometric problem of finding 
an optimal representation of multivariate data in a low dimension with respect to mean squared error. 
 Assume that the data consist of $\bx_i \in \mathbb{R}^d, i = 1, \ldots, n$. 
To project the original $d$-dimensional data into lower dimensions, say, $k<d$, we represent each $\bx$ by $\bmu + \bU\bU^T(\bx-\bmu)$, where $\bmu \in \mathbb{R}^d$ and
$\bU$ is a $d \times k$ matrix with orthonormal columns. 
\cite{pearson1901pca}  showed that the minimum of the mean squared error of the $k$-dimensional representation, 
\begin{equation}
 \mathop{\min}_{\bmu \in \mathbb{R}^d, ~ \bU^T\bU=\bI_k} \sum_{i=1}^n  \| \bx_i - \bmu - \bU\bU^T(\bx_i - \bmu) \|^2 \label{eq:mseopt} 
\end{equation}
is attained when $\bmu$ equals the sample mean and $\bU$ is a matrix with the first $k$ eigenvectors of the sample covariance matrix.

\label{sec:2b}

The MSE criteria is closely linked to a Gaussian assumption. Borrowing the characterization of PCA in \cite{CollinsEtAl2001}, suppose that $\bx_i \sim \mathcal{N}(\btheta_i, \bI_d)$, and 
$\btheta_i$ are constrained to lie in a $k$-dimensional subspace.
That is, $\btheta_i$ are in the span of a  $k$-dimensional orthonormal basis $\{\textbf{b}_l \in \mathbb{R}^d, l=1,\ldots,k\}$ so that $\btheta_i = \sum_{l=1}^k a_{il} \textbf{b}_l$ for some $a_{il}$. In this case, the negative log-likelihood is proportional to 
\begin{equation}
\label{eq:svd}
\sum_{i=1}^n \| \bx_i - \btheta_i \|^2  = \| \bX - \bA\bB^T \|_F^2,
\end{equation}
where $\bX$ is the $n \times d$ matrix with $\bx_i^T$ in the $i$th row, $\bA$ is an $n \times k$ matrix with elements $a_{il}$, $\bB$ is a $d \times k$ matrix made up of the $k$ basis vectors, and 
$\| \cdot \|_F$ is the Frobenius norm. It is well known \citep{eckart1936svd} that 
this objective function is minimized by the rank $k$ truncated singular value decomposition (SVD) of $\bX$, where $\bB$ consists of the first $k$ right singular vectors and $\bA$ consists of
 the first $k$ left singular vectors scaled by the first $k$ singular values or, equivalently, $\bA=\bX\bB$. 
The fact that the columns in $\bB$ span the same subspace as the first $k$ eigenvectors of $\bX^T\bX$ yields
the equivalence in the solutions of PCA and SVD. Note that $\bx_i$ in \eqref{eq:svd} is assumed to be mean-centered. Otherwise, $\bx_i$ is to be replaced with $\bx_i - \bmu$ as in \eqref{eq:mseopt}.

Using the equivalence and the alternative formulation of PCA through the mean space approximation under a Gaussian distribution,
\cite{CollinsEtAl2001} extended PCA to exponential family data. 
Such an extension would be useful for handling  binary, count, or non-negative data that abound in practice
since the Gaussian assumption may be inappropriate. 
The authors proposed to extend PCA to exponential family distributions in a similar way that generalized linear models (GLMs) \citep{mccullagh1989glm} extend linear regression to response variables that are non-Gaussian. Assuming that $\bx_1, \ldots, \bx_n$ are a sample from an exponential family distribution with corresponding natural parameters $\btheta_{1},\ldots,\btheta_{n}$, the goal of generalized PCA is to minimize the negative log-likelihood, subject to the constraint that the estimated natural parameters belong to a $k$-dimensional subspace. This is done by approximating  the $n \times d$ matrix of the natural parameters $\bTheta=[\theta_{ij}]$ by factorization of the form  $\bTheta = \bA \bB^T$,  where $\bA$ and $\bB$ are of rank $k$. 

Based on this generalization of PCA, \cite{ScheinEtAl2003} and \cite{Deleeuw2006} have developed algorithms specifically for binary data. \cite{ScheinEtAl2003} also extended the specification of the natural parameter space in \cite{CollinsEtAl2001} by introducing variable main effects or biases, $\bmu$, so that $\bTheta = \bOne_n \bmu^T+\bA \bB^T$. In an extension of probabilistic PCA \citep{tipping1999probabilistic}, \cite{Tipping1998} proposed a factor model for binary data with a similar structure, where the case-specific factors in $\bA$ are assumed to come from a $k$ dimensional standard normal distribution. The marginal distribution of $\bX$ is then used to define a likelihood of $\bB$ and $\bmu$ and is maximized with respect to $\bB$ and $\bmu$. 
Like other methods, estimating principal component scores on new data requires additional computation to solve for the parameters on the new cases.
More recently, \cite{lee2010sparse} extended sparse PCA to binary data by adding a sparsity-inducing $L_1$ penalty to $\bB$. In recommendation systems, \cite{johnson2014logistic} showed that logistic PCA can capture the latent features of a binary matrix more efficiently than matrix factorization with a squared error loss.


\label{sec:2c}

\section{New Formulation of Generalized PCA}
\label{sec:formulation}

We propose a new formulation of generalized PCA and demonstrate its conceptual and computational advantages over the current formulation. For the new formulation, we begin with a new interpretation of standard PCA as a technique for low dimensional projection of the natural parameters of the {\it saturated model}, which are the same as the data under the normal likelihood.

To elaborate on this perspective, recall that for a data matrix $\bX$, the $d \times k$ matrix with the first $k$ principal component loading vectors minimizes
\begin{equation}
\label{eq:samp_pca} 
\sum_{i=1}^n  \|\bx_i - \bmu - \bU\bU^T(\bx_i - \bmu)\|^2 = \|\bX - \bOne_n\bmu^T - (\bX - \bOne_n\bmu^T)\bU\bU^T \|^2_F
\end{equation}
among all $d\times k$ matrices $\bU$ such that $\bU^T\bU=\bI_k$. The principal component loadings are given by the first $k$ eigenvectors of the sample covariance matrix. 
To draw the connection to the Gaussian model and GLMs, assume that $x_{ij}$ are normally distributed with known variance. With the identity function as the canonical link, the natural parameter $\theta_{ij}$ is the mean itself in this case, and the saturated model, which is the best possible fit to the data, is the model with natural parameters
$\satnat_{ij}=x_{ij}$.

For Gaussian data, the natural parameters of the saturated model are equal to the data and minimizing the squared error in  \eqref{eq:samp_pca}  is equivalent to minimizing the deviance. Hence, standard PCA can be viewed as a technique for minimizing the Gaussian deviance by projecting the natural parameters of the saturated model into a lower dimensional space.

\subsection{Alternative formulation to logistic PCA}
\label{sec:3a}

When the data are binary, assume instead that $x_{ij}$ are from Bernoulli($p_{ij}$). The natural parameter for the Bernoulli distribution is $\theta_{ij} = \logit ~ p_{ij}$.
Let $\satnat$ represent the natural parameters of the saturated model. The saturated model occurs when $p_{ij}=x_{ij}$, which means that
\begin{equation*}
\satnat_{ij} =
  \begin{cases}
   -\infty  & \text{if } x_{ij} = 0 \\
   \infty   & \text{if } x_{ij} = 1
  \end{cases}.
\end{equation*} 

To apply an equivalent principal component analysis to binary data, we need to minimize the Bernoulli deviance by projecting the natural parameters of the saturated model onto a $k$-dimensional space. For convenience, define $q_{ij}=2x_{ij}-1$, which converts the binary variable from taking values in $\{0,1\}$ to $\{-1,1\}$. Let $\bQ=2\bX-\bOne_n \bOne_d^T$ be the matrix with elements $q_{ij}$. For practical purposes, we will approximate $\satnat_{ij}$ by $m\cdot q_{ij}$ for a large number $m$ and show how the choice of $m$ affects the analysis in Section  \ref{sec:4sim:Probs}. Therefore, $\Satnat=m \bQ$ approximates the matrix of natural parameters for the saturated model. 

Define $D(\bX; \bTheta)$ as the deviance of estimated natural parameter matrix $\bTheta$ with the data matrix $\bX$. As in standard PCA, the natural parameters are estimated with a matrix of the form $\bTheta = \bOne_n \bmu^T + (\Satnat - \bOne_n \bmu^T) \bU \bU^T$. The objective function to minimize is the Bernoulli deviance, 
\begin{align*}
D(\bX; \bOne_n \bmu^T + (\Satnat - \bOne_n \bmu^T) \bU \bU^T) & = -2 \left( \log p(\bX; \bOne_n \bmu^T + (\Satnat - \bOne_n \bmu^T) \bU \bU^T) - \log p(\bX; \Satnat) \right) \\
& =  -2 \langle \bX, \bOne_n \bmu^T + (\Satnat - \bOne_n \bmu^T) \bU \bU^T \rangle \\
& + 2 \sum_{i=1}^n \sum_{j=1}^d \log \left( 1 + \exp(\mu_j + [\bU\bU^T (\bsatnat_i - \bmu)]_{j}) \right), \numberthis \label{eq:lpca}
\end{align*}
subject to $\bU^T\bU=\bI_k$, where $\langle \bA, \bB \rangle = tr(\bA^T \bB)$ is the trace inner product.

\subsection{Generalized PCA formulation}
As alluded to in the previous section, this methodology can be extended to any distribution in the exponential family. If $x$ comes from a one-parameter exponential family distribution, then the deviance $D(x; \theta) $ is proportional to $\{ -x \theta + b(\theta) + c(x)\}$,
where $\theta$ is the canonical parameter, and $E(X) = b^\prime(\theta)$. Let $g(\cdot)$ be the canonical link function, such that $g(b^\prime(\theta)) = \theta$. The saturated model with the lowest possible deviance occurs when $\satnat = g(x)$.

For a given distribution, using the appropriate $ b(\theta)$, $g(\cdot)$, and $\satnat$, generalized PCA can be cast as 
\eq{
\mathop{\min}_{\bmu \in \mathbb{R}^d,~ \bU^T\bU = \bI_k} - \langle \bX, \bOne_n \bmu^T + (\Satnat - \bOne_n \bmu^T) \bU \bU^T \rangle 
 + \sum_{i=1}^n \sum_{j=1}^d b \left(\mu_j + [\bU\bU^T (\bsatnat_i - \bmu)]_{j} \right).
}
For example, for the Bernoulli distribution $b(\theta) = \log(1 + \exp(\theta))$ and $\satnat = \logit~x$, for Poisson $b(\theta) = \exp(\theta)$ and $\satnat = \log x$, and for Gaussian $b(\theta) = \theta^2/2$ and $\satnat = x$.

Further, this formulation can handle matrices with multiple types of data. Each column of the saturated natural parameter matrix will correspond to the particular member of the exponential family and $b_j(\theta)$ will change by column.

\subsection{Comparison to previous techniques}
\label{sec:3b}

The main advantage of the proposed formulation is that we only solve for the principal component loadings and not simultaneously for the principal component scores. The previous method for logistic PCA posits that the logit of the probability matrix, logit $\bP$, can be represented by a low-rank matrix factorization
\begin{equation*}
\bTheta = \bA \bB^T,
\end{equation*}
assuming $\bmu = \bZero$ here to simplify exposition.
Our formulation, on the other hand, assumes the logit of the probability matrix has the form 
\begin{equation*}
\bTheta = \Satnat \bU \bU^T.
\end{equation*}

To highlight the difference between the two formulations, we will call the previous formulation logistic SVD (LSVD) and our formulation logistic PCA (LPCA). 
The $d \times k$ principal component loading matrices for LSVD and LPCA are $\bB$ and $\bU$, respectively, and the $n \times k$ matrices of principal component scores are $\bA$ and $\Satnat \bU$.
The loading matrices are comparable, but the score matrices take different forms. The form of $\Satnat \bU$, along with $m$, can act as an implicit regularizer. 
While there is no restriction on how large the elements of $\bA$ can be, the elements of $\Satnat \bU$ are bound between $- m \sqrt{d}$ and $m \sqrt{d}$, for instance.

We illustrate a number of advantages of the alternative formulation. Conceptually, when the main effects are not included in logistic SVD, the cases and variables are treated interchangeably. That is, an analysis of $\bX$ will produce the same low-rank fitted matrix as an analysis of $\bX^T$, and the loadings of $\bX$ will equal the scores of $\bX^T$ and vice versa. Logistic PCA, however, will very likely have a different solution for the respective loadings and scores. Since the intent of PCA is typically to explain the dependence of the variables, we believe that it is desirable to maintain the inherent difference between cases and variables in the analysis.

\label{sec:3b:NewData}
Another difference is that the proposed formulation allows quick and easy evaluation of principal component scores for new data.
Let $\bx^* \in \{0,1\}^d$ be a new observation. We wish to calculate the principal component scores for this new data point assuming that the loadings have already been estimated using another dataset. Let $\hat{\bB}$ and $\hat{\bU}$ be the principal component loadings estimated from the LSVD and LPCA formulations, respectively. To determine the new principal component scores, $\ba^* \in \mathbb{R}^k$, for LSVD, one needs to find
\begin{equation}
\ba^* = \mathop{\argmin}_{\ba \in \mathbb{R}^k} \sum_{j=1}^d \left[-x^*_j \ba^T \hat{\textbf{b}}_{j \cdot} + \log\left( 1 + \exp(\ba^T \hat{\textbf{b}}_{j \cdot}) \right)\right],
\end{equation}
where $\hat{\textbf{b}}_{j \cdot}$ is the $j$th row of $\hat{\bB}$. This is equivalent to a logistic regression problem, where the $d$-dimensional response vector is $\bx^*$ and the $d \times k$ design matrix is $\hat{\bB}$. The $\ba^*$ can be viewed as the coefficient vector that maximizes the likelihood defined through $\bx^*$.

In contrast, the LPCA formulation only requires a matrix multiplication for the new principal component scores:
\begin{equation*}
\hat{\bU}^T \bsatnat^* ,
\end{equation*}
where $\bsatnat^* := m(2\bx^*-1) := m\bq^*$ is taken as the approximate natural parameters for
$\bx^*$ under the saturated model. This process is analogous to computing the principal components for new data in standard PCA, where $\bx^*$ itself acts as $\bsatnat^*$, a set of the natural parameters for the saturated model. 
Further, predicting a low-rank estimate of the natural parameters on a set of new data only requires calculating
\begin{equation}
\hat{\bTheta}^* = \Satnat^* \hat{\bU} \hat{\bU}^T. \label{eq:pred_param}
\end{equation}

Quick evaluation of principal component scores for new data can be particularly useful in a number of situations. For example, principal component regression (or classification) \citep[see][\S 3.5]{ESL} can be extended to logistic PCA when all the covariates are binary. If it is necessary to make predictions for a large amount of new data or predictions are required in real-time, the LSVD method may be too slow. Further, our proposed method will be much more efficient when the number of principal components to retain is chosen by cross validation \citep[\S 6.1.5]{jolliffe2005principal}. In this case, we select the number of components that best reconstruct the dataset on held-out observations, and the cross-validation requires applying logistic PCA to new data repeatedly.

\label{sec:3b:LessParams}

Another major difference between this formulation and the previous one is that the alternative formulation entails much fewer parameters. In particular, LSVD has $kn - \frac{k(k-1)}{2}$ additional parameters in the $\bA$ matrix, if the columns are constrained to be orthogonal to each other. This additional number of parameters could potentially be very large and be ripe for over-fitting. As \cite{welling2008deterministic} discussed, logistic SVD can be viewed as an estimation method for a factor model. Instead of marginalizing over the case-specific factors, which are latent variables in factor analysis, logistic SVD takes a degenerate approach to computing point estimates for them. Since the number of latent factors is proportional to the number of observations, overfitting can easily occur.

\label{sec:3d:m} 
The alternative formulation of logistic PCA does have an additional parameter, $m$, that previous formulations do not. 
We treat $m$ as a tuning parameter. As $m$ gets larger, the elements of $\Satnat$ will get closer to $\pm \infty$ and therefore the estimated probabilities will be close to 0 or 1. Conversely, if $m$ is small, the probability estimates will be close to $0.5$. If the user has domain knowledge of the range of the likely probabilities, they can use this to guide their choice of $m$. We have found that cross validation is an effective way to choose $m$. Simulations in Section \ref{sec:4sim} show the potential benefits of correctly choosing $m$.

Because LSVD has many more parameters than LPCA given a rank, LSVD is guaranteed to have a lower in-sample deviance. Despite this, the simulations in Section \ref{sec:4sim} show that LPCA can do just as well or better at estimating the true probabilities if $m$ and $k$ are chosen properly. Further, the loadings learned in LSVD may not generalize as well as those learned in LPCA, as the example in Section \ref{sec:overfit} exhibits.


\subsection{Number of Principal Components}
\label{sec:3d}
\label{sec:3d:PCs}

Selecting the appropriate dimensions for effective data representation is a common issue for dimensionality reduction techniques. There has been relatively little discussion previously in the literature of how to select the number of PCs in logistic PCA. \cite{lee2010sparse} derived a BIC heuristic to select the degree of sparsity for sparse logistic PCA and \cite{li2010simple} proposed a Bayesian version of exponential family PCA with a prior on the loadings that controls the number of principal components. We propose a few methods for selection of dimensionality in logistic PCA, motivated by the current practices in standard PCA and the dual interpretation of squared error as the deviance for a Gaussian model. 

One common approach in standard PCA is to look at the cumulative percent of the variance explained and select the number of components such that a chosen proportion, $\gamma$, is met or exceeded. Let $\hat{\bU}_k$ be the rank $k$ estimate of the principal component loadings. The criteria will choose a rank $k$ model if $k$ is the smallest integer such that
\[1 - \|\bX - (\bOne_n \hat{\bmu}^T + (\bX - \bOne_n \hat{\bmu}^T) \hat{\bU}_k \hat{\bU}_k^T)\|_F^2 / \|\bX - \bOne_n \hat{\bmu}^T\|_F^2 > \gamma.\]

Similarly for logistic PCA, if $D(\bX; \hat{\bTheta}_k)$ is the Bernoulli deviance of the rank-$k$ principal component loadings, $\hat{\bTheta}_k = \bOne_n \hat{\bmu}^T + (\Satnat - \bOne_n \hat{\bmu}^T) \hat{\bU}_k \hat{\bU}_k^T$, with the data $\bX$, then we could choose the smallest integer $k$ such that
\begin{equation*}
1-\frac{D(\bX; \hat{\bTheta}_k)}{D(\bX; \bOne_n \hat{\bmu}^T)} > \gamma.
\end{equation*}
This criterion has a similar interpretation as in standard PCA that at least $100 \gamma$\% of the deviance is explained by $k$ principal component loadings. Notice that, as expected, 100\% of the deviance will be explained by $d$ components because $\hat{\bU}_d \hat{\bU}_d^T=\bI_d$ and $D(\bX; \Satnat)=0$ by definition.

Another approach from standard PCA is to create a scree plot of the percent of variance explained by each component and look for an elbow in the plot. The same analogy can be made to logistic PCA, however with a modified definition of the percent of reduction in deviance for additional components. For logistic PCA, the principal component loadings matrices are not necessarily nested, meaning the first $k-1$ columns of $\hat{\bU}_k$ do not necessarily equal $\hat{\bU}_{k-1}$. For the reason, it would be more appropriate to define the {\it marginal} percentage of deviance explained by the additional $k$th component as
\begin{equation*}
\frac{D(\bX; \hat{\bTheta}_{k-1})-D(\bX; \hat{\bTheta}_k)}{D(\bX; \bOne_n \hat{\bmu}^T)}.
\end{equation*}
If $d$ is large, these non-sequential procedures could potentially take very long. We recommend only calculating these quantities until the cumulative percent of deviance explained is fairly high, which may be much smaller than $d$ in many situations.

\subsection{Geometry of the Projection}
\begin{figure}[t]
	\centering
		\includegraphics[scale=0.56]{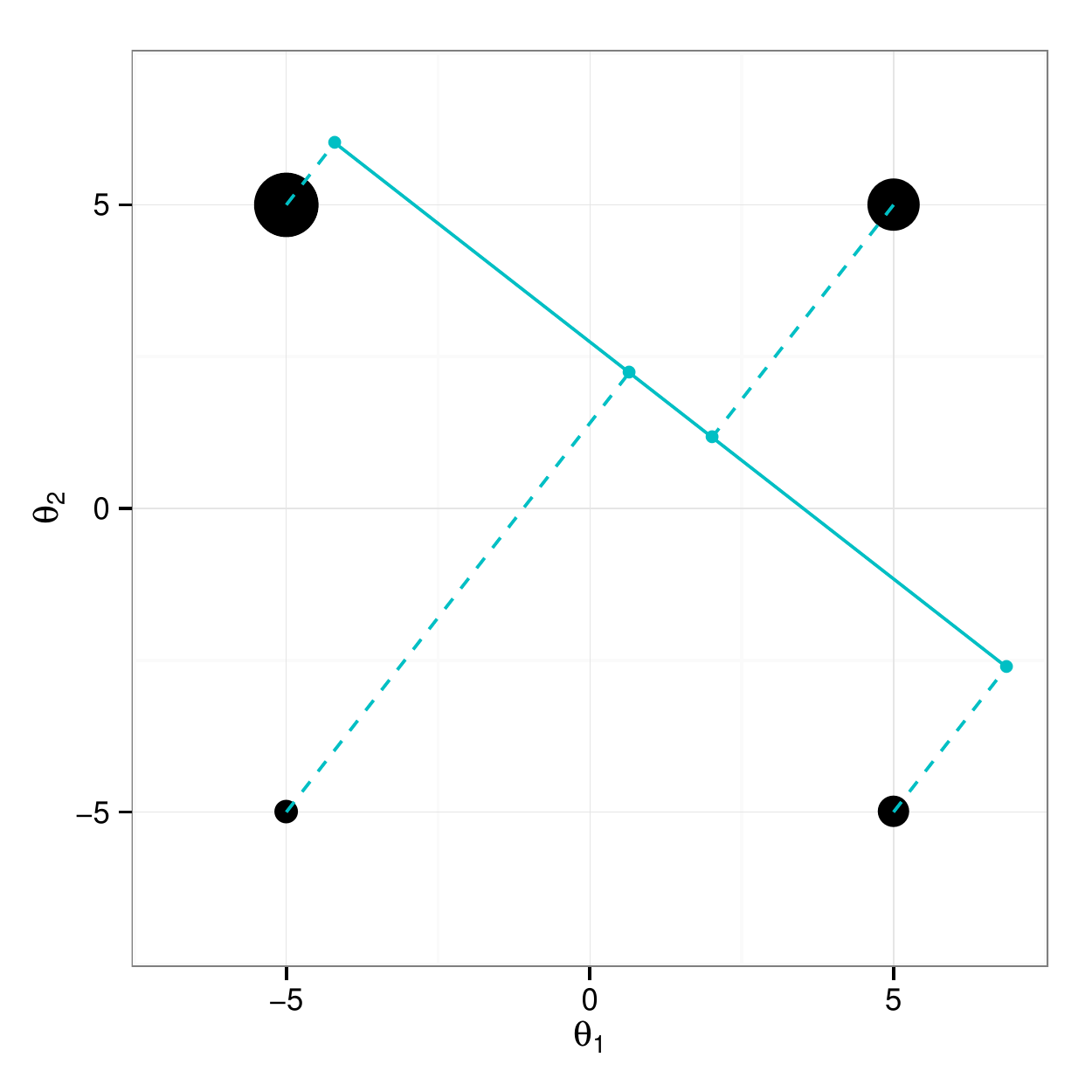}\includegraphics[scale=0.56]{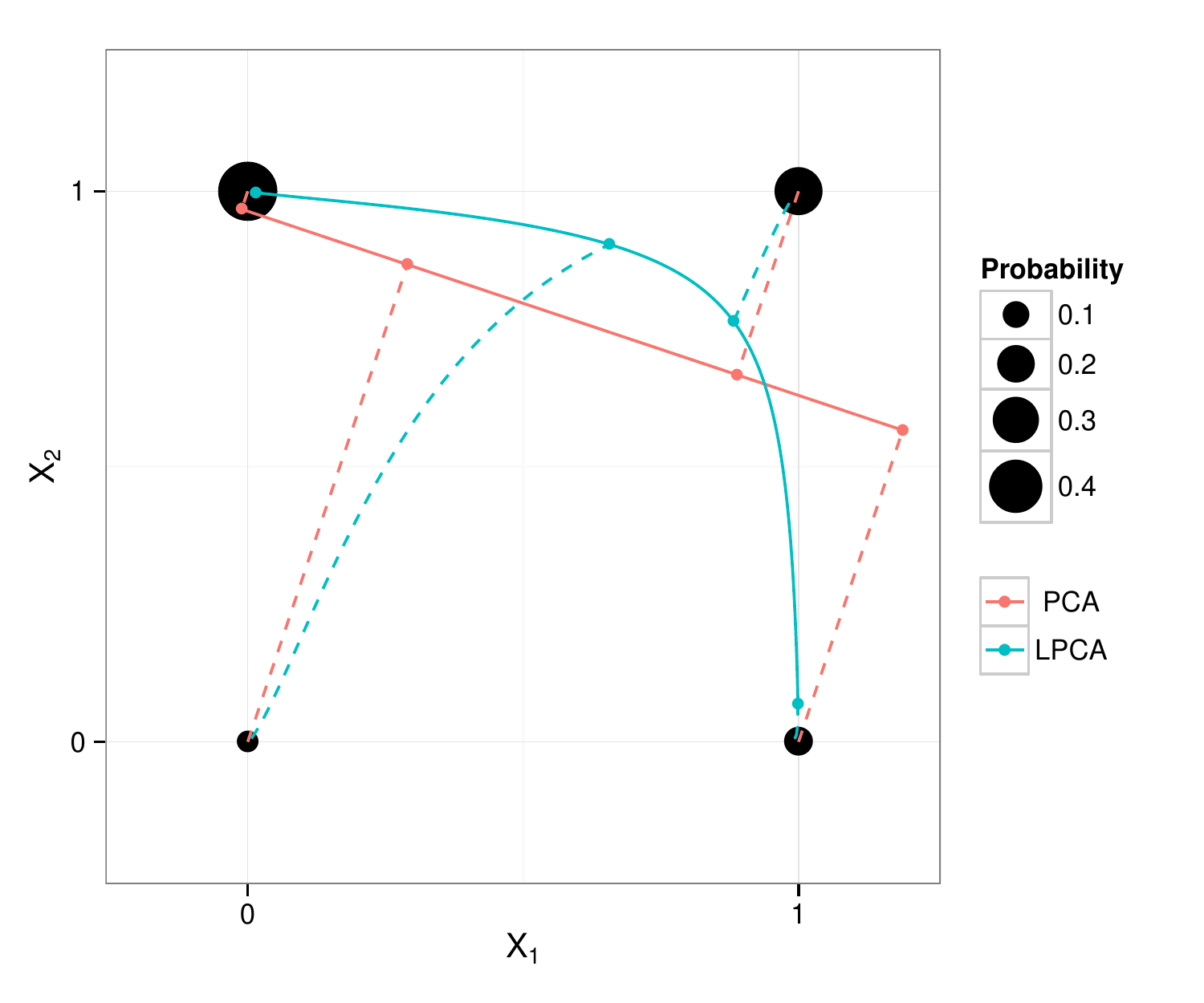}
	\caption{Logistic PCA projection of a two-dimensional Bernoulli distribution in the natural parameter space (left) and in the probability space (right) with $m=5$ compared to the PCA projection}
	\label{fig:lpca_proj}
\end{figure}

The geometry of logistic principal components is illustrated with a distribution of two Bernoulli random variables.
With two dimensions ($d = 2$), there are only four possible responses as indicated by the four points in the right panel of
Figure \ref{fig:lpca_proj}.  The areas of the points in the figure are proportional to the probabilities that are randomly assigned to them. The probabilities sum to one and specify the distribution, which is taken as a population version of data. 

The first step of the proposed logistic PCA is to transform the data into the natural parameters of the saturated model, which is what is shown on the left panel of Figure \ref{fig:lpca_proj}. Here, $m = 5$ was used to represent the saturated model. The one-dimensional linear projection is then performed in the natural parameter space, and is chosen to minimize the Bernoulli deviance. In this example, we minimized the deviance by a grid search since there is only one free parameter for the projection. The solid line in the left panel is the rank-one space for logit parameters, the dots on the line are the projected values, and the dashed lines show the correspondence between the saturated model parameters and their projections.

Next, we transform this linear projection to the probability space by taking the inverse logit of the values, which is shown as the blue line in the right panel of Figure \ref{fig:lpca_proj}. The one-dimensional space for the probabilities induced by logistic PCA is now non-linear and the projections are no longer orthogonal. Also shown on this plot is the  one-dimensional space (red line) from standard PCA. The projections from standard PCA are not constrained to be between 0 and 1, as is seen for the projections of $(1, 0)$ and $(0, 1)$.  In addition, the fitted probabilities from logistic PCA seem to be closer to the original data values than standard PCA for all except $(0,0)$, which has the smallest  probability.

\section{Logistic PCA for Patterned Data}
\label{sec:patterned}
The properties of standard PCA solutions are well understood algebraically. Under certain assumptions, the solutions are explicitly known. For example, when the variables are uncorrelated, the loadings are the standard bases and the principal components are ordered from highest variance to lowest. In contrast, not much is known about the solutions of logistic PCA or logistic SVD. To obtain analogous results for patterned data with logistic PCA, we derive necessary conditions for the solutions first, and find solutions that satisfy (or nearly satisfy) these optimality conditions under different sets of assumptions on data matrices. These results help us gain a better understanding of logistic PCA.

\subsection{First-Order Optimality Conditions}
\label{sec:3c:Grad}
To enforce orthonormality $\bU^T\bU=\bI_k$, we add a constraint to the objective in \eqref{eq:lpca} via the method of Lagrange multipliers. The Lagrangian is 
\eq{
L(\bU, \bmu,  \Lambda) = D(\bX; \bOne_n \bmu^T + (\Satnat - \bOne_n \bmu^T) \bU \bU^T) + tr \left( \Lambda (\bU^T\bU - \bI_k) \right),
}
where $\Lambda$ is a $k \times k$ symmetric matrix of Lagrange multipliers \citep{wen2013ortho}. 

Taking the gradient of the Lagrangian with respect to $\bU$, $\bmu$, and $\Lambda$ and setting them equal to $\bZero$, we obtain the first-order optimality conditions for the solution of logistic PCA:
\begin{flalign}
 & ~ \left[\left(\bX - \hat{\bP}\right)^T \left(\Satnat - \bOne_n \bmu^T \right) + \left(\Satnat - \bOne_n \bmu^T \right)^T \left(\bX - \hat{\bP}\right) \right] \bU = \bU \Lambda & \label{eq:lpca_deriv} \\
 & ~ \left(\bI_d - \bU \bU^T \right) \left( \bX - \hat{\bP} \right)^T \bOne_n  =  \bZero_d \text{, and} \label{eq:lpca_deriv_mu} \\
 & ~~ \bU^T \bU = \bI_k. \label{eq:ortho}
\end{flalign}
The matrix $\hat{\bP}$ has the estimate of $p_{ij}$ at $\bU$ and $\bmu$, $\hat{p}_{ij} = \sigma \left( \mu_j + [\bU\bU^T (\bsatnat_i - \bmu)]_{j} \right)$, as its $ij$th element, where $\sigma(\theta)=\frac{\exp(\theta)}{1 + \exp(\theta)}$ is
the inverse logit (or sigmoid) function.
The details of the calculation can be found in Appendix \ref{app:deriv_lpca}. For the following sections, let 
\[\bC^m := \left[\left(\bX - \hat{\bP}\right)^T \left(\Satnat - \bOne_n \bmu^T \right) + \left(\Satnat - \bOne_n \bmu^T \right)^T \left(\bX - \hat{\bP}\right) \right],\] which is labeled to explicitly state the dependence of $\Satnat$ and $\hat{\bP}$ on $m$.

Applying the Lagrangian method of multipliers to the standard PCA formulation \eqref{eq:mseopt} yields $\left[ (\bX - \bOne_n \bmu^T)^T (\bX - \bOne_n \bmu^T)\right] \bU = \bU \Lambda$ as part of the first-order optimality conditions, which is very similar to the form of equation \eqref{eq:lpca_deriv}. Unlike standard PCA, \eqref{eq:lpca_deriv} is nonlinear in $\bU$ and the solution is not known in closed form through an eigen-decomposition because the matrix in the left hand side of \eqref{eq:lpca_deriv} depends on $\bU$ through $\hat{\bP}$.
However, we can derive some explicit results for special cases using the optimality conditions.

\subsection{Independence}
There are a few natural extremes of data dependence that a given binary dataset can exhibit.
On one end of the spectrum, all of the columns can be in the span of a single vector. In this case, all the columns are equal to each other, and the dataset can be reconstructed with arbitrary precision with a rank-one approximation.

The opposite end of the spectrum is when the variables are independent of each other. In standard PCA, this implies that the covariance matrix is diagonal, so the principal component loadings are the standard basis vectors. Analogous to standard PCA results, we show below that, if the $l$th column of a dataset is uncorrelated with the other $d - 1$ columns and its column mean is $\frac{1}{2}$ then the $l$th standard basis vector, $\be_l$, satisfies the first-order optimality conditions of logistic PCA with $k = 1$. If its column mean is not equal to $\frac{1}{2}$, $\be_l$ can nearly satisfy the optimality conditions with large enough $m$.

Let $X_j$ be the length $n$ vector of the $j$th column of $\bX$ and $\bar{X}_j = \bOne_n^T X_j / n$ be the corresponding column mean.

\begin{theorem} \label{thm:ind}
Assume that $X_l^T X_j / n = \bar{X}_l \bar{X}_j$, for all $j \neq l$, i.e. the $l$th variable is uncorrelated with all other variables.
\begin{itemize}
\item[(i)] If $\bar{X}_l = \frac{1}{2}$, then $\bu = \be_l$, the $l$th standard basis vector, satisfies the first-order optimality conditions of logistic PCA, regardless of $m$. That is,
\eq{
\bC^m \be_l - \lambda_m \be_l = \bZero,
}
for some $\lambda_m$.

\item[(ii)] If $\bar{X}_l \neq \frac{1}{2}$, then the first-order optimality conditions can be satisfied as close as desired with $\bu = \be_l$ for $m$ large enough. Formally, for any $\epsilon>0$, there exists $m_0$ such that, for all $m > m_0$,
\eq{
\| \bC^m \be_l - \lambda_m \be_l \|^2 < \epsilon,
}
with $\lambda_m = 0$.
\end{itemize} 
\end{theorem}

The proof is given in Appendix \ref{sec:ind_proof}. If multiple columns are uncorrelated with the remaining columns, this result easily generalizes to larger $k$. For example, if  $k$ columns are uncorrelated with all other columns, then a rank $k$ solution comprising of the the corresponding $k$ standard basis vectors can be made arbitrarily close to (or exactly if the column means equal $\frac{1}{2}$) satisfying the necessary conditions \eqref{eq:lpca_deriv}--\eqref{eq:ortho}.

This leads to a natural question: when there are multiple candidate solutions, which one decreases the deviance the most?

\begin{theorem} \label{thm:order}
For logistic PCA with $k = 1$, the standard basis vector which decreases deviance the most is the one corresponding to column with mean closest to $\frac{1}{2}$.
\end{theorem}

The proof is given in Appendix \ref{sec:order_proof}. This result corresponds to the ordering of variables by variance in standard PCA. The variables with means closest to $\frac{1}{2}$ have the largest variance. If the variables are independent, the variance explained will be largest with a standard basis vector corresponding to the variable with largest variance. Similar to the previous theorem, this theorem can be easily extended to $k$ larger than 1. In this case, the loading matrix made up of the the standard basis vectors corresponding to the $k$ columns that are closest to $\frac{1}{2}$ will decrease the deviance the most out of all loading matrices that comprise of standard basis vectors.

\subsection{Compound symmetry}
With independence and perfect correlation being two extremes of the structure of the data, somewhere in the middle is compound symmetry. A covariance matrix $\Sigma$ is compound symmetric if the diagonals are constant ($\Sigma_{jj} = c_1$ for all $j$) and the off-diagonals are all equal to each other ($\Sigma_{jk} = c_2$ for all $j \neq k$). The compound symmetry of $\Sigma$ implies equal correlations among the variables. If this is the case, $\frac{1}{\sqrt{d}}\bOne_d$ is an eigenvector of the covariance matrix. We show that, under more limiting conditions, $\frac{1}{\sqrt{d}}\bOne_d$ satisfies the optimality conditions for logistic PCA when $\bQ^T\bQ$ is compound symmetric. 

$\bQ^T\bQ$ has a natural interpretation. The diagonals always equal $n$ and the $jk$th off-diagonal is the number of records in which the $j$th and $k$th variables are the same minus the number of records in which the $j$th and $k$th variables differ. It measures how much the $j$th and $k$th variables agree with each other, and can range from $-n$ (total disagreement) to $n$ (total agreement). $\bQ^T\bQ$ is compound symmetric if all the bivariate agreements are the same.

\begin{theorem} \label{thm:compound}
Consider logistic PCA without main effects, where $\bmu = \bZero$. Assume that $\bQ^T\bQ$ is compound symmetric. If $\beta$'s exist such that the following condition is satisfied,
\begin{equation}
\sigma \left( \frac{m}{d} \sum_{l \not \in \{j,k\}} q_{il} \right) = \frac{1}{2} + \sum_{l \not \in \{j,k\}} q_{il} \beta_{jk,l}, ~ \mbox{for all } 
j \ne k, ~j, k =1,\cdots,d, \mbox{ and } i=1,\cdots,n
\label{eq:cs_cond}
\end{equation}
then $\bu = \frac{1}{\sqrt{d}}\bOne_d$ satisfies the first-order optimality conditions for logistic PCA, as characterized by equations \eqref{eq:lpca_deriv}--\eqref{eq:ortho}.
\end{theorem}

The proof is given in Appendix \ref{sec:cs_proof}. Equation \eqref{eq:cs_cond} states that the fitted probabilities can be represented as an affine function of the $q_{ij}$s. When there are fewer columns,
the condition is more likely to be met.

\begin{corollary}
If $d \leq 4$ and $\bQ^T\bQ$ is compound symmetric, then \eqref{eq:cs_cond} is satisfied and therefore $\bu = \frac{1}{\sqrt{d}}\bOne_d$ satisfies the first-order optimality conditions. When $d = 2$, $\bQ^T\bQ$ is always compound symmetric.
\end{corollary}

\begin{proof}
For $d = 4$, without loss of generality, let $j = 1$ and $k = 2$. 
\eq{
\sigma \left( \frac{m}{4} (q_{i3} + q_{i4}) \right) = 
  \begin{cases}
   \sigma(m / 2)  & \text{if } q_{i3} = q_{i4} = 1 \\
   \frac{1}{2}  & \text{if } q_{i3} \neq q_{i4}\\
   \sigma(- m / 2)   & \text{if } q_{i3} = q_{i4} = -1
  \end{cases}.
}
Since $\sigma(- m / 2) = 1 - \sigma(m / 2)$, 
\eq{
\sigma \left( \frac{m}{4} (q_{i3} + q_{i4}) \right) = \frac{1}{2} + q_{i3} \beta_{12,3} + q_{i4} \beta_{12,4}
}
where
\eq{
\beta_{12,3} = \beta_{12,4} = \frac{\sigma(m / 2) - 1/2}{2}.
}
The same can be shown for $d = 2$ or $d = 3$. 
\end{proof}

%

\section{Computation}
\label{sec:computation}

Optimizing for the principal component loadings of logistic PCA is difficult because of the non-convex objective function and the orthonormality constraint. We derive two algorithms for generating solutions. The first is guaranteed to find a local solution and the second finds a global solution to a relaxed version of logistic PCA. 

\subsection{Majorization-minimization (MM) algorithm}
\label{sec:3c:MM}
One approach to minimizing the deviance is to iteratively minimize simpler objectives.
Majorization-minimization \citep[][]{hunter2004MM} seeks to solve difficult optimization problems by majorizing the objective function with a simpler objective, and minimizing the majorizing function. The majorization function must be equal to or greater than the original objective for all inputs, and equal to it at the current input value. 



The deviance of a single estimated natural parameter $\theta$ is quadratically approximated at $\theta^{(t)}$ by
\eq{
-2 \log p(x; \theta) & = -2 x \theta + 2 \log(1 + \exp(\theta)) \\
& \approx -2 x \theta^{(t)} + 2 \log(1 + \exp(\theta^{(t)})) + 2 (\hat{p}^{(t)} - x) (\theta - \theta^{(t)}) + \hat{p}^{(t)} (1 - \hat{p}^{(t)})(\theta - \theta^{(t)})^2 \\
& \leq -2 x \theta^{(t)} + 2 \log(1 + \exp(\theta^{(t)})) + 2 (\hat{p}^{(t)} - x) (\theta - \theta^{(t)}) + \frac{1}{4}(\theta - \theta^{(t)})^2,
}
where $\hat{p}^{(t)} = \sigma(\theta^{(t)})$. The inequality is due to the variance of a Bernoulli random variable being bounded above by $1/4$. \cite{Deleeuw2006} showed that the deviance itself is majorized by this same function.


Therefore, the deviance for the whole matrix is majorized by
\eq{
& \sum_{i,j} \left\{ -2 x_{ij} \theta_{ij}^{(t)} + 2 \log(1 + \exp(\theta_{ij}^{(t)})) + 2 (\hat{p}^{(t)}_{ij} - x_{ij}) (\theta_{ij} - \theta_{ij}^{(t)}) + \frac{1}{4}(\theta_{ij} - \theta_{ij}^{(t)})^2 \right\} \\
= & \frac{1}{4} \sum_{i,j} (\theta_{ij} - z_{ij}^{(t)})^2 + C
}
where $C$ is a constant that does not depend on $\theta_{ij}$ and 
\begin{equation}
 z^{(t)}_{ij} = \theta^{(t)}_{ij} + 4 [x_{ij} - \sigma(\theta_{ij}^{(t)}) ] \label{eq:working}
\end{equation}
are the working variables in the $t$th iteration. Further, let $\bZ^{(t)}$ be a matrix whose $ij$th element equals $z^{(t)}_{ij}$. The working variables have a similar form to the so-called adjusted response used in the iteratively reweighted least squares algorithm for generalized linear models \citep{mccullagh1989glm}. Instead of having weights equal to the estimated variance at the current estimates, we use the upper bound, which allows for minimization of the majorization function.

The logistic PCA objective function can be majorized around estimates of $\bU^{(t)}$ and $\bmu^{(t)}$ as
\begin{align*}
D(\bX; \bOne_n \bmu^T + (\Satnat - \bOne_n \bmu^T) \bU \bU^T)  & \leq \frac{1}{4} \sum_{i,j} \left( \mu_j + [\bU\bU^T (\bsatnat_i - \bmu)]_{j} -  z^{(t)}_{ij} \right)^2 + C,
\end{align*}
and hence, the next iterates of $\bU^{(t+1)}$ and $\bmu^{(t+1)}$ can be obtained by minimizing
\begin{equation}
\sum_{i,j} \left( \mu_j + [\bU\bU^T (\bsatnat_i - \bmu)]_{j} -  z^{(t)}_{ij} \right)^2  =  \|\bOne_n \bmu^T + (\Satnat - \bOne_n \bmu^T) \bU \bU^T - \bZ^{(t)} \|_F^2. \label{eq:maj_iter}
\end{equation}

Given initial estimates for $\bU$ and $\bmu$, a solution can be found by iteratively minimizing equation \eqref{eq:maj_iter}, subject to orthonormality constraint $\bU^T \bU=\bI_k$. With fixed $\bmu$, 
the minimizer of the majorizing function can be found by expanding equation (\ref{eq:maj_iter}). Letting $\Satnat_c := \Satnat - \bOne_n \bmu^T$ and $\bZ^{(t)}_c := \bZ^{(t)} - \bOne_n \bmu^T$, then
\eq{
\argmin_{\bU^T\bU=\bI_k} \|\Satnat_c \bU \bU^T - \bZ_c^{(t)} \|_F^2 & = \argmin tr \left( \bU\bU^T \Satnat_c^T \Satnat_c \bU\bU^T \right) - tr \left( \bU\bU^T \Satnat_c^T  \bZ^{(t)}_c \right) - tr \left( (\bZ^{(t)}_c)^T \Satnat_c \bU \bU^T \right) \\ 
& = \argmin tr \left[ \bU^T \left( \Satnat_c^T \Satnat_c - \Satnat_c^T  \bZ^{(t)}_c - (\bZ^{(t)}_c)^T \Satnat_c \right) \bU \right] \\
& = \argmax tr \left[ \bU^T \left( \Satnat_c^T  \bZ^{(t)}_c + (\bZ^{(t)}_c)^T \Satnat_c - \Satnat_c^T \Satnat_c \right) \bU \right].
}
%
The trace is maximized when $\bU$ consists of the first $k$ eigenvectors of the symmetric matrix $\Satnat_c^T  \bZ^{(t)}_c + (\bZ^{(t)}_c)^T \Satnat_c - \Satnat_c^T \Satnat_c$ \citep{fan1949eig}.

With fixed $\bU$, the minimization of the majorizing function with respect to $\bmu$ is a least squares problem, and the majorizing function is minimized by 
\eq{
\bmu = \frac{1}{n} (\bZ^{(t)} - \Satnat \bU \bU^T)^T \bOne_n,
}
which can be interpreted as the average differences between the projection of the uncentered saturated natural parameters and the current working variables. Algorithm \ref{alg:lpca_mm} presents the MM algorithm for logistic PCA.

\begin{algorithm}[t]
 \label{alg:lpca_mm}
 \SetKwInOut{Input}{input}\SetKwInOut{Output}{output}
 \SetAlgoLined
 \Input{Binary data matrix ($\bX$), $m$, number of principal components ($k$), convergence criteria ($\epsilon$)}
 \Output{$d \times k$ orthonormal matrix of principal component loadings ($\bU$) and column main effects ($\bmu$)}
 \BlankLine
 Set $t=0$ and initialize $\bmu^{(0)}$ and $\bU^{(0)}$. We recommend setting $\mu^{(0)}_j = \logit \bar{X}_j$ and setting $\bU^{(0)}$ to the first $k$ right singular vectors of $\bQ$ \\
 \Repeat{Deviance converges}{
 	\begin{enumerate}
 	\item $t \leftarrow t+1$
 	\item Set the working variables \\$z_{ij}^{(t)} = \theta_{ij} + 4[x_{ij} - \sigma(\theta_{ij}) ]$ where $\theta_{ij} = \mu^{(t-1)}_j + [\bU^{(t-1)} (\bU^{(t-1)})^T (\bsatnat_i - \bmu^{(t-1)})]_j$
 	\item $\bmu^{(t)} = \frac{1}{n} (\bZ^{(t)} - \Satnat \bU^{(t-1)} (\bU^{(t-1)})^T)^T \bOne_n$
 	\item Carry out the eigen decomposition of $(\Satnat - \bOne_n (\bmu^{(t)})^T)^T  (\bZ^{(t)} - \bOne_n (\bmu^{(t)})^T) + (\bZ^{(t)} - \bOne_n (\bmu^{(t)})^T)^T (\Satnat - \bOne_n (\bmu^{(t)})^T) - (\Satnat - \bOne_n (\bmu^{(t)})^T)^T (\Satnat - \bOne_n (\bmu^{(t)})^T) = E \Lambda E^T$ and \\
 	set $\bU^{(t)}$ to the first $k$ eigenvectors in $E$
 	\end{enumerate}
 }
 \caption{Majorization-minimization algorithm for logistic PCA}
\end{algorithm}


Since the optimization problem is non-convex, finding the global minimum is difficult in general. Due to the fact that the majorization function is above the deviance and tangent to it at the previous iteration, minimizing the majorization function at each iteration must either decrease the deviance or cause no change at each iteration. While the majorization-minimization does not guarantee a global minimum generally, the quadratic majorization function and the smooth objective does guarantee finding a local minimum \citep[Chapter 12]{lange2013optimization}.

\subsection{Convex relaxation}
\label{sec:3c:Fan}
Another approach to solving for the loadings of logistic PCA is to relax the non-convex domain of the problem to a convex set. Instead of solving over rank-$k$ projection matrices $\bU\bU^T$, we can optimize over the convex hull of rank-$k$ projection matrices, called the Fantope \citep{dattorro2005convex}, which is defined as 
\begin{equation*}
\mathcal{F}^k = \text{conv}\{ \bU\bU^T ~|~ \bU \in \mathbb{R}^{d \times k},~\bU^T\bU=\bI_k \} = \{ \bH~|~\bZero \preceq \bH \preceq \bI_d,~tr(\bH)=k \}.
\end{equation*}
With the Fantope, $k$ no longer needs to be an integer, but can be any positive number.

The logistic PCA problem then becomes
\begin{flalign*}
\begin{aligned}
 \mbox{minimize} &~~ D(\bX; \bOne_n \bmu^T + (\Satnat - \bOne_n \bmu^T) \bH) & \\
 \mbox{subject to} &~~~ \bmu \in \mathbb{R}^d \text{ and } \bH \in \mathcal{F}^k.
\end{aligned}
\end{flalign*}
For fixed $\bmu$, this is a convex problem in $\bH$ since the objective and the constraints are both convex. Similarly, for fixed $\bH$, this is convex in $\bmu$. Hence, the problem is bi-convex.

With fixed $\bH$, minimizing with respect to $\bmu$ is similar to a GLM problem and many strategies can be used. We choose to solve for $\bmu$ first and leave it fixed while $\bH$ is solved for. With $\bH = \bZero$, the deviance is minimized with respect to $\bmu$ when $\mu_j = \logit \bar{X}_j$. We do not update $\bmu$ again in our implementation, but other strategies, such as an alternating gradient descent, could easily be implemented if desired.

Once $\bmu$ is fixed, to solve the relaxed problem with respect to $\bH$, we use accelerated gradient descent \citep{nesterov2007gradient}, which can be used to solve {\it composite} convex problems of the form
\begin{equation*}
f(\bH) + h(\bH)
\end{equation*}
where $f(\bH)$ is convex and differentiable and $h(\bH)$ is convex. Further, the gradient of $f(\bH)$ with respect to $\bH$ must be Lipschitz continuous, i.e. there exists $L$ such that for all $\bH_1$ and $\bH_2$
\eq{
\| \nabla f(\bH_1) - \nabla f(\bH_2) \|_F \leq L \| \bH_1 - \bH_2 \|_F,
}
where $\nabla f(\bH_1)$ is the gradient of $f(\bH)$ with respect to $\bH$, evaluated at $\bH_1$. For our problem, $f(\bH) = D(\bX; \bOne_n \bmu^T + (\Satnat - \bOne_n \bmu^T) \bH)$ and $h(\bH) = 1_{\{\bH \in \mathcal{F}^k\}}$, the convex indicator function which is 0 if $\bH$ is in the Fantope set and $\infty$ otherwise.

At the $t$th iteration, the accelerated gradient descent algorithm takes the steps
\eq{
\textbf{F}^{(t)} & = \bH^{(t-1)} + \frac{t - 2}{t + 1} (\bH^{(t-1)} - \bH^{(t-2)}) \\
\bH^{(t)} & = \Pi_{\mathcal{F}^k}(\textbf{F}^{(t)} - \frac{1}{L} ~ \nabla f(\textbf{F}^{(t)})),
}
where $\Pi_{\mathcal{F}^k}(\cdot)$ is the Euclidean projection operator onto the convex set $\mathcal{F}^k$. 

This algorithm is very similar to projected gradient descent \citep{boyd2003subgradient}, but there is an extra step to adjust the previous iteration $ \bH^{(t-1)}$ to $\textbf{F}^{(t)}$. This adjustment allows accelerated gradient descent to be a first-order method that achieves $\mathcal{O}(1/t^2)$ convergence rate \citep{beck2009fista}. That is,
\eq{
\left(f(\bH^{(t)}) + h(\bH^{(t)})\right) - \left(f(\bH^{*}) + h(\bH^{*})\right) \leq \frac{2L ||\bH^{(0)} - \bH^*||_F^2}{(t + 1)^2},
}
where $\bH^{*}$ is the value that minimizes the objective.

To use this method, we must derive the gradient of the deviance with respect to $\bH$, show that it is Lipschitz continuous, and derive the projection operator $\Pi_{\mathcal{F}^k}(\cdot)$.

Without taking into account the symmetry of $\bH$, the gradient of the deviance with respect to $\bH$ is
\begin{equation}
\frac{\partial D}{\partial \bH} = 2 \left(\hat{\bP}-\bX\right)^T \left(\Satnat - \bOne_n \bmu^T \right), \label{eq:fan_grad}
\end{equation}
which is derived in a similar way to (\ref{eq:lpca_deriv}). Since $\bH$ is symmetric, the gradient needs to be symmetrized as follows (see \cite{peterson2012matrix}):
\begin{equation}
\nabla f(\bH) = \left[\frac{\partial D}{\partial \bH}\right] + \left[\frac{\partial D}{\partial \bH}\right]^T - \mbox{diag}\left[\frac{\partial D}{\partial \bH}\right]. \label{eq:fan_grad_sym}
\end{equation}

\begin{theorem} \label{thm:lipschitz}
The gradient of the Bernoulli deviance $D(\bX; \bOne_n \bmu^T + (\Satnat - \bOne_n \bmu^T) \bH)$ with respect to $\bH$ is Lipschitz continuous with Lipschitz constant 
\eq{
L = \|\Satnat - \bOne_n \bmu^T \|_F^2.
}
\end{theorem}

The proof is given in Appendix \ref{sec:lipschitz_proof}.

To update $\bH$, the projection operator $\Pi_{\mathcal{F}^k}$ has to be defined explicitly. \cite{vu2013fantope} showed that the Euclidean projection of a positive definite matrix $\textbf{M}$ with a spectral decomposition $\textbf{M} = \sum_{j=1}^d \lambda_j \bu_j \bu_j^T$ onto $\mathcal{F}^k$ is given by
\begin{equation}
\Pi_{\mathcal{F}^k} (\textbf{M}) = \sum_{j=1}^d \lambda_j^+(\nu) \bu_j \bu_j^T, \label{eq:fan_proj}
\end{equation}
where $\lambda_j^+(\nu) = \min(\max( \lambda_j-\nu,0),1)$ with $\nu$ that satisfies $\sum_j \lambda_j^+(\nu) = k$. A quick and easy way to solve for $\nu$ is by bisection search, with an initial lower bound of $\lambda_d-k/d$ and an initial upper bound of $\lambda_1$, the largest eigenvalue. Algorithm \ref{alg:lpca_fan} presents the complete algorithm for the Fantope solution.

\begin{algorithm}[t]
 \label{alg:lpca_fan}
 \SetKwInOut{Input}{input}\SetKwInOut{Output}{output}
 \SetAlgoLined
 \Input{Binary data matrix ($\bX$), $m$, rank ($k$)}
 \Output{$d \times d$ rank-$k$ Fantope matrix ($\bH$)}
 \BlankLine
 Set $t=0$, $\mu_j = \logit \bar{X}_j$, $j = 1, \ldots, d$, and $L = \|\Satnat - \bOne_n \bmu^T \|_F^2$ \\
 Initialize $\bH^{(-1)} = \bH^{(0)} = \bU\bU^T$, where $\bU$ consists of the first $k$ right singular vectors of $\Satnat - \bOne_n \bmu^T$ \\
 \Repeat{Deviance converges}{
 	\begin{enumerate}
 	\item $t \leftarrow t+1$
 	\item Set $\textbf{F}^{(t)} = \bH^{(t-1)} + \frac{t - 2}{t + 1} (\bH^{(t-1)} - \bH^{(t-2)})$ 
 	\item Calculate the gradient, $\nabla(\textbf{F}^{(t)})$, using equations \eqref{eq:fan_grad} and \eqref{eq:fan_grad_sym}
 	\item Move in the negative direction of the gradient and project onto the \\
 	rank-$k$ Fantope using equation \eqref{eq:fan_proj},
 	$\bH^{(t)} = \Pi_{\mathcal{F}^k} \left(\textbf{F}^{(t)} - \frac{1}{L} \nabla(\textbf{F}^{(t)}) \right)$
 	\end{enumerate}
 }
 \caption{Fantope algorithm for convex relaxation to logistic PCA}
\end{algorithm}

\subsubsection{Discussion of the Convex Formulation}

Even though this algorithm is guaranteed to find a global minimum of the deviance over the Fantope space, it is of practical interest how to translate the estimated $\hat{\bH}$ into an estimated loadings matrix $\hat{\bU}$ because we may be more interested in principal component loadings and the scores they induce. For that reason, once the algorithm has converged, we may convert the Fantope matrix $\hat{\bH}$ into a projection matrix by setting $\hat{\bU}$ equal to the matrix with the first $k$ eigenvectors of $\hat{\bH}$. We will discuss the impact of this conversion in Section \ref{sec:4sim:FantopeMM}.

Both the Fantope solution and the projection solution have the previously-stated advantage that applying the principal components analysis to new data only involves matrix multiplications. There are some practical differences however. If the goal of the analysis is to have a small number of principal components which can be interpreted or used for subsequent analyses, the Fantope solution may not be preferred. Even for $k=1$, it is possible for the solution to be of full rank, no matter $d$. 

However, the Fantope algorithm is useful if viewed as a constrained method of approximating a binary matrix, when interpretation is not desired. This type of approximation can be useful when accurate estimates of the fitted probabilities are required, for example, in collaborative filtering. Further, since the problem is convex, it can also be solved more quickly and reliably than the formulation with a projection matrix. In fact, warm starts may be employed to solve for the optimal matrices over a fine grid of $k$ and $m$. 

Another property of the convex relaxation is given below.

\begin{theorem}
For any fixed $k$, the convex relaxation of logistic PCA can fit any dataset arbitrarily well with large enough $m$.
\end{theorem}
\begin{proof}
Let $\bmu = \bZero$ and $\bH = (k/d)~\bI_d$, which satisfies $\bH \in \mathcal{F}^k$. The fitted value of the natural parameters is
\eq{
\hat{\bTheta} = \Satnat \bH = \frac{m k}{d} \bQ.
}
With fixed $k$ and $d$, $\hat{\bTheta}$ will tend towards the saturated model parameters, $\Satnat$, as $m \rightarrow \infty$, which gives the perfect fit to the data.
\end{proof}

In this sense, $m$ can be seen as a tuning parameter that controls the extent to which the estimates fit the data. However, unlike traditional regularization parameters, the above ``over-fitted'' model will fit a new case, $\bx^*$, just as well as the training data. One way to combat this is to set elements of the matrix to missing and choose the $m$ which completes the matrix most accurately, a topic of future research.

This behavior of $m$ is similar to that of $k$ in standard PCA, logistic PCA, and its convex counterpart, where increasing $k$ towards $d$ necessitates a better fit of the data, and likely a better fit on new observations as well.

In the non-convex formulation, we have empirically found that the deviance does not always decrease with increasing $m$. In fact, the deviance will decrease to zero only if $\text{sign}(\hat{\theta}_{ij}) = q_{ij}$ for all $i$ and $j$, which is unlikely if $k$ is much smaller than $d$. 

An R \citep{R2015} implementation of both algorithms for logistic PCA, as well as an algorithm for logistic SVD, can be found at \url{cran.r-project.org/web/packages/logisticPCA/}.

\section{Numerical Examples}
\label{sec:examples}
In this section, we demonstrate how the logistic PCA formulation works on both simulated and real data. In particular, we numerically examine the differences between the previous formulation and the proposed formulation, the effect of $m$, and the effectiveness of the algorithms.

\subsection{Simulation}
\label{sec:4sim}
For comparison of our formulation of logistic PCA (LPCA) with the previous formulation (LSVD), we simulate binary matrices from a family of multivariate Bernoulli mixtures that induce a low-rank structure in both the true probability matrix and the logit matrix. The components or clusters of the Bernoulli mixtures are determined by cluster-specific probability vectors.
 		
\subsubsection{Simulation setup}
\label{sec:4sim:Setup}
Let $k$ indicate the number of mixture components or clusters of the probability vectors for $d$-variate binary random vectors making up an $n\times d$ data matrix $\bX$. For a specified $n$, $d$, and $k$, let $C_i$ be the cluster assignment for the $i$th observation, $i=1,\ldots,n$, which takes one of $k$ values \{$1,\ldots, k$\} with equal probability. 
Further, we specify $k$ true probability vectors, $\boldsymbol{p}_c \in [0,1]^d$,  for $c=1,\ldots,k$. In our simulations, the probabilities are independently generated from a $\text{Beta}(\alpha,\beta)$ distribution and $X_{ij}$ given $C_i=c$ are independently generated from Bernoulli$(p_{cj})$. For a Beta distribution, instead of the shape parameters $\alpha$ and $\beta$, we vary  the mean $\bar{p}=\frac{\alpha}{\alpha+\beta}$ and
concentration parameter $\phi=\frac{\alpha+\beta}{2}$, which is inversely related to  the variance of $\bar{p}(1-\bar{p})/(2\phi+1)$.  If we let $\bA$ be an $n \times k$ indicator matrix with $A_{ic}=1_{\{ C_i=c \}}$ and $\bB_*$ be a $d \times k$ matrix with the $c$th column equal to $\boldsymbol{p}_c$, then the true probability matrix is 
\begin{equation*}
\bP=[p_{ij}]=\bA \bB_*^T,
\end{equation*}
or equivalently, the logit matrix is $\bTheta=\logit~ \bP = \bA \bB^T$ with $\bB=\logit~\bB_*$. The accuracy of the approximation of $\bP$ through logistic PCA depends on the number of principal components considered and the value of parameter $m$.
To reduce confusion between the true number of clusters $k$ and the number of principal components considered, we will adopt the notation $\hat{k}$ for the  number of principal components considered. For all simulations in this section, we set $n=100$ and $d=50$. For simulations with $\bar{p} = 0.5$, we did not include the main effects $\bmu$ for either LPCA or LSVD. We did this to minimize the differences in implementations and the main effects are likely close to $\bZero$ when $\bar{p} = 0.5$.

\subsubsection{Fantope versus MM}
\label{sec:4sim:FantopeMM}
To compare the performance of the MM and Fantope algorithms, a dataset was simulated using a two-component ($k=2$) mixture of probabilities from a beta distribution with $\bar{p}=0.5$ and $\phi=3$. We then estimated the logistic principal component loadings with $\hat{k}=2$ and $m=4$ fifteen times with both the Fantope algorithm and the MM algorithm, each with a random initialization for $\bU^{(0)}$ or $\bH^{(0)}=\bU^{(0)}(\bU^{(0)})^T$. 
For the Fantope, the starting step size was determined by averaging the inverse of the element-wise second derivatives evaluated at the initial value of $\bH^{(0)}$.
For each iteration, we kept track of the average deviance of the estimates, which is defined as $D(\bX; \hat{\bTheta})/(nd)$. Figure \ref{fig:FantopeMM} displays the average deviance of a sequence of the solutions from the algorithms, each line representing a different initialization. We also kept track of the deviance of the projection matrix estimate from the Fantope estimate for each iteration, which we labeled ``Fantope-Proj'' in the plot. For this estimate, we use the first $\hat{k}$ eigenvectors of $\bH$ to estimate $\bU$.
\begin{figure}[t]
	\centering
		\includegraphics{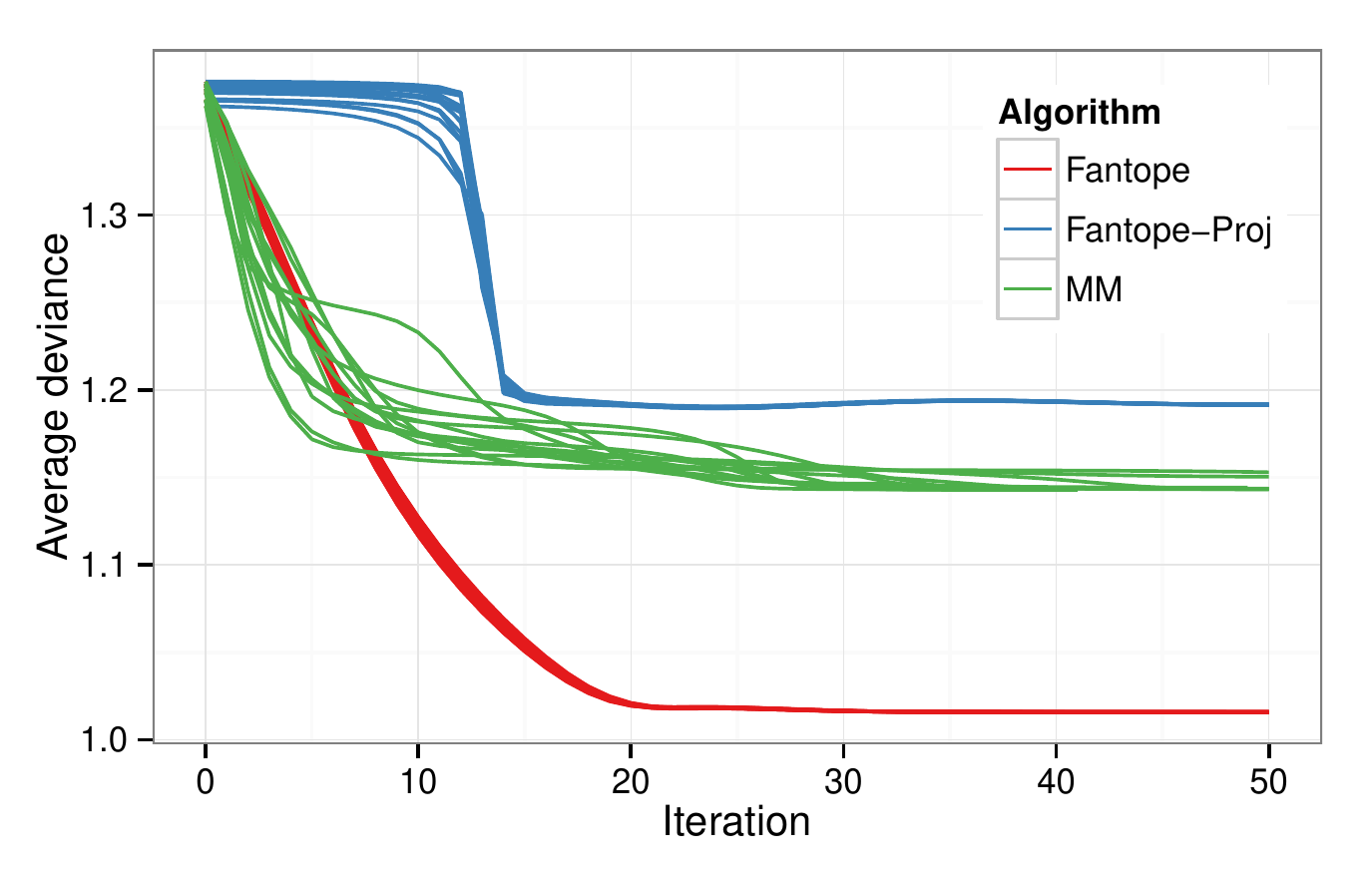}
	\caption{Average deviance by iteration number for solutions from Fantope and MM algorithms 
       with 15 different initial values on a simulated dataset}
	\label{fig:FantopeMM}
\end{figure}

There are several notable insights in this figure. The deviance for the Fantope solution fairly quickly converges to the same minimum value, regardless of initialization. As expected, the Fantope solution provides a deviance lower bound, while the deviance for the projection matrix resulting from the Fantope solution does not necessarily decrease in the same manner. In fact, the MM algorithm finds solutions better than the ``projected'' Fantope solution for every initialization. The MM algorithm is more sensitive to initialization, but the range of the deviance values of solutions is fairly narrow in this example. 

This numerical comparison suggests that when the projection matrix $\bU\bU^T$ is concerned, the MM algorithm performs better than the Fantope algorithm followed by projection. For this reason, we use the MM algorithm only in the subsequent analyses. However, it could potentially be useful to run both the MM and Fantope algorithms. The Fantope solution gives a lower bound of the deviance for the global solution with a projection matrix while the MM solution and Fantope-Proj solution provide upper bounds. The gap between the lower and upper bounds could give some indication of our confidence in the projection solution. If there is a small gap, we can be fairly sure that the projection solution is close to the global minimum. 

\subsubsection{Estimate of true probabilities}
\label{sec:4sim:Probs}
\begin{figure}[t]
	\centering
		\includegraphics[scale=0.75]{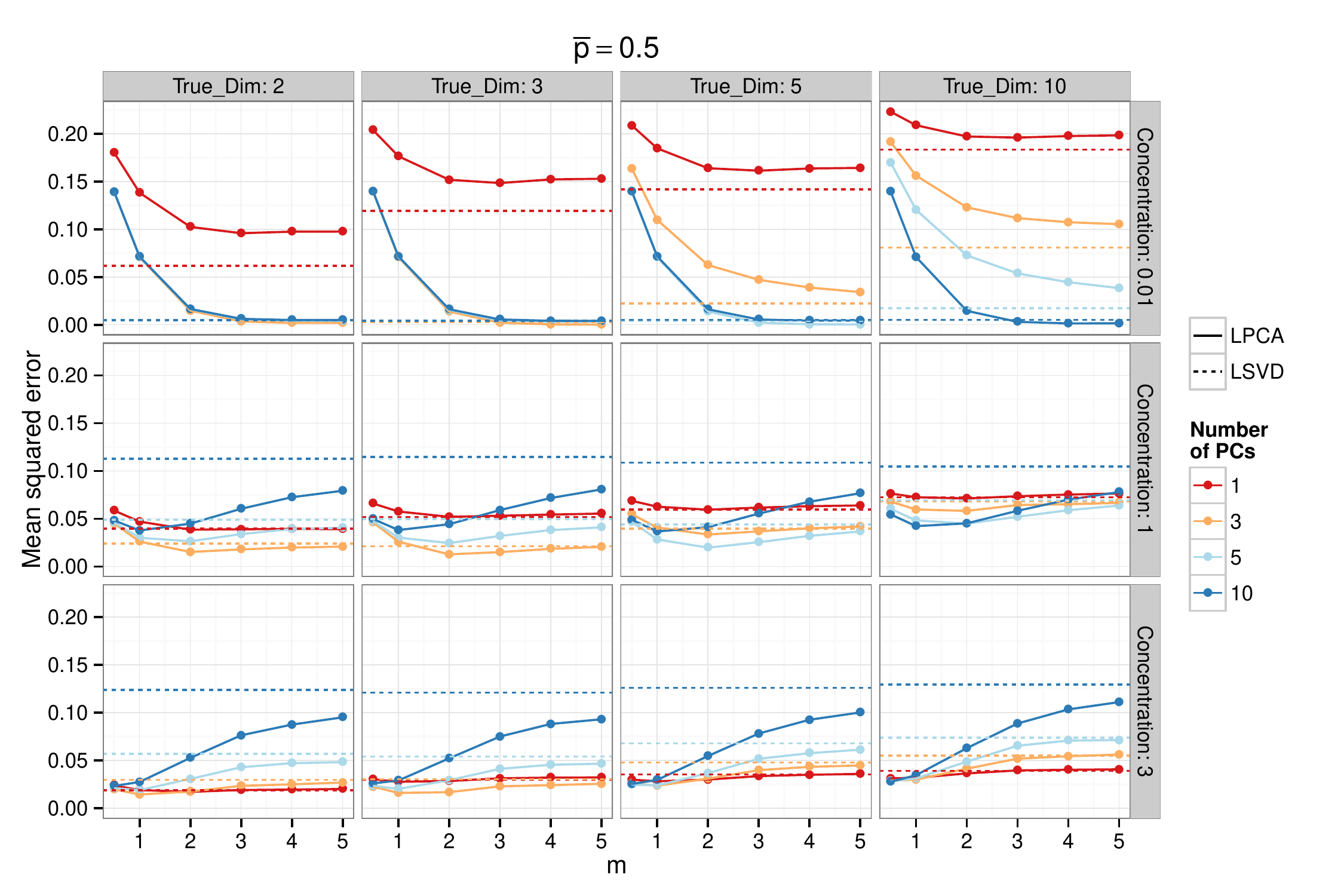}
	\caption{Mean squared error of probability estimates derived from logistic PCA with varying $m$ and logistic SVD in a simulation experiment where the rank of the true probability matrix ranges from 2 to 10 and the extent of concentration of the true probabilities varies from low to high}
	\label{fig:PSqerror05}
\end{figure}
We compare our formulation of logistic PCA (LPCA) with previous formulations (LSVD) and study the effect of the additional $m$ parameter on the probability estimates. For the numerical study, we simulated 12 binary matrices from a variety of different cluster models with $\bar{p}=0.5$. We varied the true number of clusters ($k \in \{2,3,5,10\}$) and the concentration parameter for the probabilities ($\phi \in \{0.01,1,3\}$). For each of the 12 data matrices, we performed dimensionality reduction with LSVD and LPCA. For both, we varied the estimated number of components ($\hat{k} \in \{1,3,5,10\}$) and for LPCA we also varied $m$ from 0.5 to 5. Again, for LPCA, we used the MM algorithm and for LSVD we used the iterative SVD algorithm from \cite{Deleeuw2006}.

After estimating the probability matrix as $\hat{\bP}$, we compared it to the true probability matrix $\bP$ by taking the element-wise mean squared error $\|\hat{\bP}-\bP\|_F^2/(nd)$. Figure \ref{fig:PSqerror05} shows the results.
The first thing to notice is the effect of the concentration parameter. When $\phi$ is low (the top row of Figure \ref{fig:PSqerror05}), most of the true probabilities are close to 0 or 1 with not much in between. In this situation, having the correct estimated dimension $\hat{k}$ is crucial for both methods and, in fact, having $\hat{k}>k$ does no harm. Further, for LPCA, higher $m$ is better for estimating the probability matrix in this situation. Both of these results are in line with our expectations. When the rank of the estimate is allowed to be higher, the estimates are able to fit the data more closely in general. In doing so, the resulting estimated probabilities will be close to 0 and 1. As stated in Section \ref{sec:3d:m}, higher values of $m$ enable estimates that are closer to 0 and 1 as well. 

On the other hand, if the concentration is high, most of the true probabilities will be close to $\bar{p}=0.5$. The bottom row of Figure \ref{fig:PSqerror05} shows the results when this is the case. The opposite conclusions are reached from this scenario. In general, having a lower rank or a lower $m$ is better. This is again as expected.

When the concentration is moderate ($\phi=1$, the middle row of Figure \ref{fig:PSqerror05}), the true probabilities are generated from a uniform distribution. For LPCA, having the correct estimated dimension is important, but so is $m$. For each of the different true dimensions, the lowest mean squared error is achieved when the estimated dimension matches the true dimension. This is not true for all values of $m$, as when $k=10$, the estimate with $\hat{k}=10$ is poor for large $m$. For each of the estimated ranks, there is a local minimum of MSE for choosing $m$. In contrast to the high and low concentration cases, the results for LSVD do not mirror those of LPCA for the same $\hat{k}$. Here, an estimated rank of $\hat{k}=3$ is best for all of the true ranks. 


Finally, for any of the given dataset simulations, there is a combination of $m$ and $\hat{k}$ for which LPCA had a mean squared error as small as or smaller than any LSVD estimate. The challenge obviously is to find the best settings and to determine the optimal combination data-adaptively in applications. We show that, for each $\hat{k}$, using five-fold cross validation is an effective way to choose $m$. To accomplish this, we randomly split the rows into five groups. For each of the groups, we perform logistic PCA with a given $m$ on all of the data except the rows in that group. We then predict the natural parameters on the held-out group of rows with the given $m$ and the fitted orthonormal matrix using \eqref{eq:pred_param} and record the predictive deviance. After looping through all five groups, the $m$ with the lowest predictive deviance is used for logistic PCA with all the rows. 

The resulting mean squared errors using the data-adaptively chosen $m$'s are reported in Figure \ref{fig:PSqerror05M}. Using this strategy, the low-rank estimates from logistic PCA give more accurate estimates for the high and middle concentration scenarios, while logistic SVD has more accurate estimates for the low concentration scenario, unless $\hat{k} \geq k$. Standard techniques, such as those mentioned in Section \ref{sec:3d:PCs}, can be used to choose $k$ for both logistic PCA and logistic SVD.

We have also run another simulation experiment for a more sparse situation with $\bar{p}=0.1$ with all the other factors kept the same. The plot of the resulting estimated errors is in Figure \ref{fig:PSqerror01} in Appendix \ref{app:sim01}. Since the column means are far away from 0.5, we have also included main effects $\bmu$ in the model.
The results are largely the same as described in this section. The only major difference is that estimates of a lower rank than the true rank perform better overall, probably due to the fact that predicting all zeros gives a decent prediction in this case.

\begin{figure}[t]
	\centering
		\includegraphics[scale=0.75]{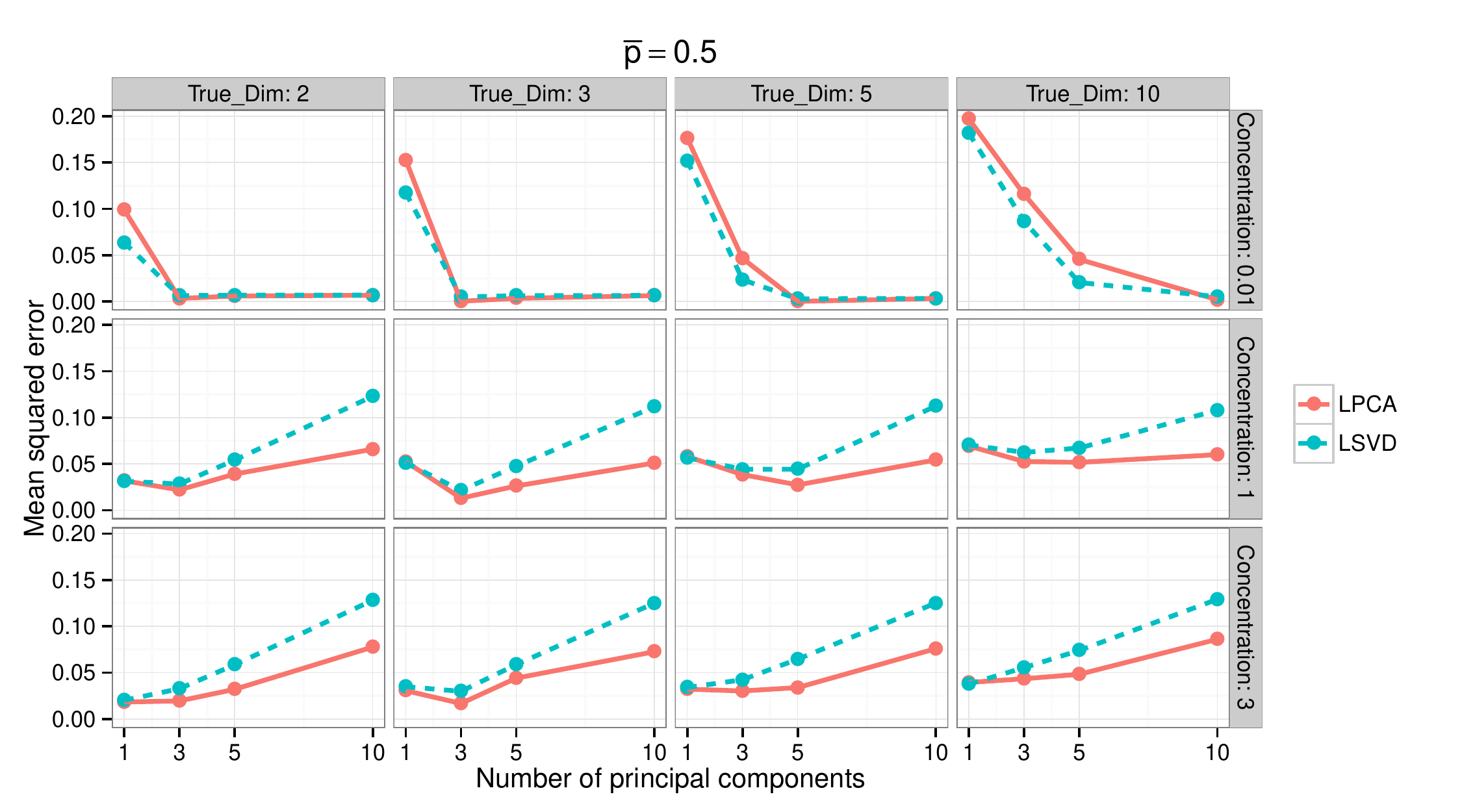}
	\caption{Mean squared error of probability estimates derived from logistic PCA and logistic SVD in the simulation experiment where $m$ in logistic PCA is chosen by cross validation}
	\label{fig:PSqerror05M}
\end{figure}

\subsection{Data Analysis}
\label{sec:4data}
We present an application of logistic PCA to patient-diagnosis data, which are part of the electronic health records data on 12,000 adult patients admitted to the intensive care units at Ohio State University's Medical Center from 2007 to 2010. Patients can be admitted to an intensive care unit (ICU) for a wide variety of reasons, some of them frequently co-occurring. While in the ICU, patients are diagnosed with one or more medical conditions of over 800 disease categories from the International Classification of Diseases (ICD-9). It is often of interest to practitioners to study the comorbidity in patients, that is, what medical conditions patients are diagnosed with simultaneously.
The latent factor view of principal components seems to be appropriate for describing the concept of comorbidity through a lower dimensional approximation of the disease probabilities or their transformations. Such a representation could reveal a common underlying structure capturing simultaneous existence of multiple medical conditions.

We analyzed a random sample of 1,000 patients. There were 584 ICD-9 codes that had at least one of the randomly-selected patients assigned to them. These patient-diagnosis data were organized in a binary matrix, $\bX$, where $x_{ij}$ is 1 if patient $i$ has disease $j$ and 0 otherwise. The proportions of disease occurrences 
ranged from 0.001 to 0.488 (the maximum corresponding to acute respiratory failure) with a median of 0.005 and the third quartile equals 0.017, meaning that most of the disease categories were rare. 
For comparison, we also applied logistic SVD and standard PCA to the data.

\subsubsection{Selecting number of principal components}
\label{sec:4data:PCs}
\begin{figure}[t]
	\centering
		\includegraphics{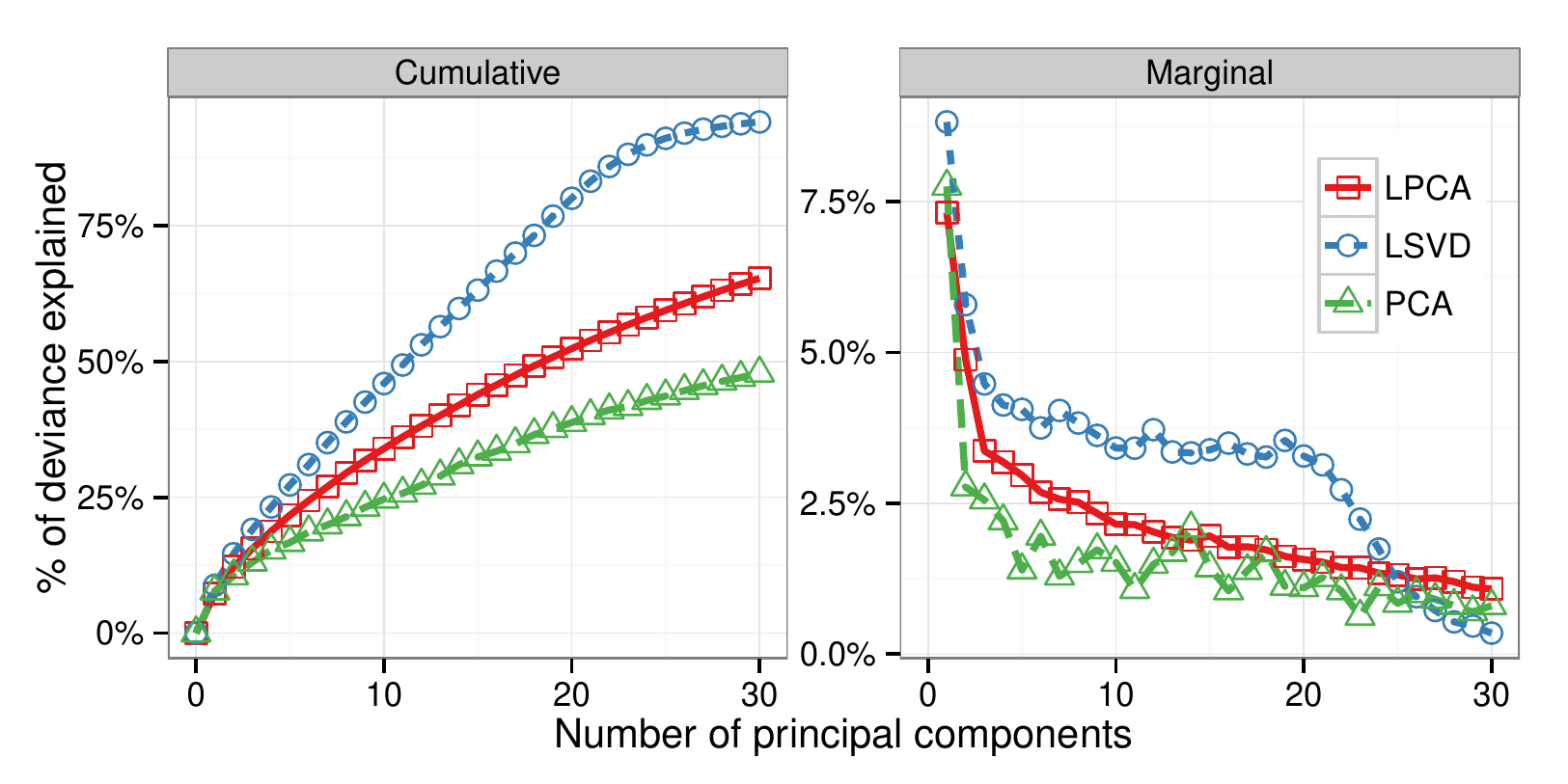}
	\caption{Cumulative and marginal percent of deviance explained by principal components of LPCA, LSVD, and PCA for the patient-diagnosis data.}
	\label{fig:MedDataNumber}
\end{figure}

As described in Section \ref{sec:3d:PCs}, to decide on the number of principal components to include in the model or approximation of the data matrix, we calculated the cumulative percent of deviance and the marginal percent of deviance explained by both LSVD and LPCA, as well as standard PCA. Figure \ref{fig:MedDataNumber}  illustrates  the change in the percent of deviance explained as the number of components increases. For logistic PCA, $m$ was chosen by five-fold cross validation. In order to have the same $m$ for all $k$ considered, we performed cross validation with $k = 15$, for which $m = 8$ had the lowest cross validation deviance. For standard PCA, we calculated the Bernoulli deviance using the reconstructed values as probability estimates. Since the reconstructed entries in the matrix could be outside the range of 0 to 1, we truncated them to be within the range $[10^{-10},1-10^{-10}]$. 

If we were using LSVD, one reasonable choice for the number of components would be 24 because that is where the cumulative percent of deviance explained begins to level off and the cumulative percent of deviance explained is quite high, at 90\%.
The marginal percent of deviance plot suggests that the marginal contributions level off after the second component for both LPCA and LSVD. 
To have a manageable number of components to analyze, we may use two components.

\subsubsection{Quality of fit}
\label{sec:4data:quality}
\begin{figure}[t]
	\centering
		\includegraphics{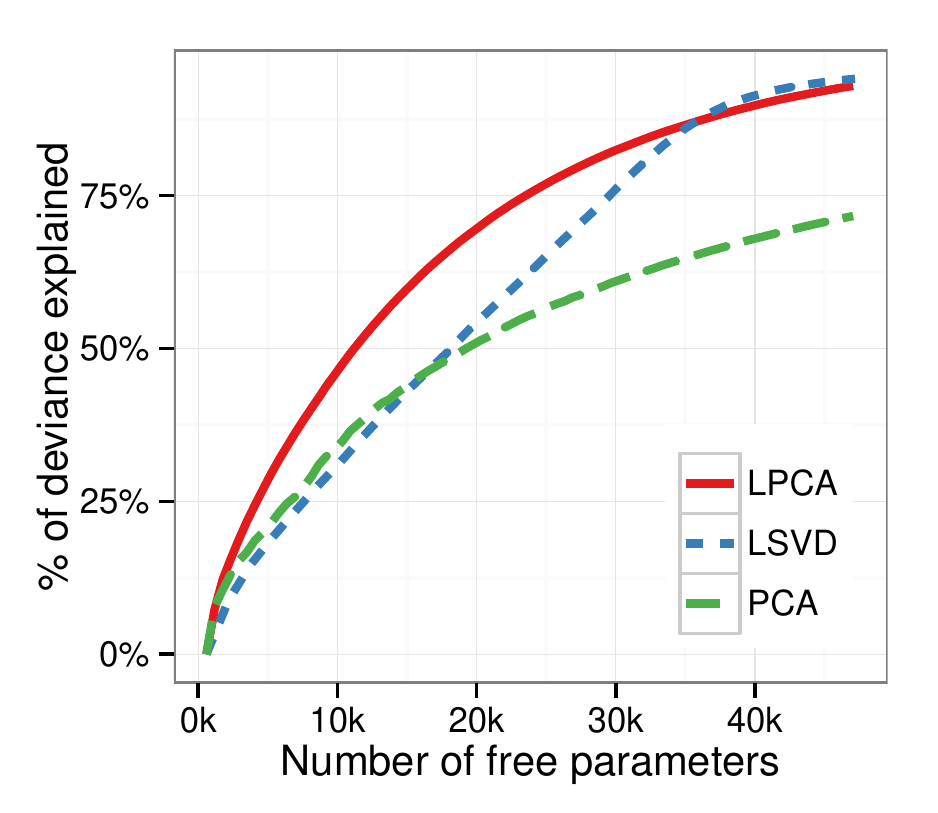}
	\caption{Cumulative percent of deviance explained by principal components of LPCA, LSVD, and PCA plotted against the number of free parameters for the patient-diagnosis data.}
	\label{fig:MedDataParams}
\end{figure}

For this dataset, it is clear that LPCA fits the data significantly better than PCA, even though they have the same number of parameters per component added. Also, LSVD has higher percent of deviance explained than LPCA after the first few components. 
However, it is not a completely fair comparison between LSVD and LPCA, because LSVD has extra parameters in $\bA$ to better fit the data. To illustrate this, we also plotted the percent of deviance explained as a function of the number of free parameters in Figure \ref{fig:MedDataParams}. When indexed by the number of free parameters, LPCA looks quite favorable compared to LSVD. In fact, there isn't a large difference between standard PCA and LSVD for smaller numbers of parameters.

To further show the advantages of LPCA, we used the loadings learned from this data to construct a low-rank estimate of the logit matrix using \eqref{eq:pred_param} for a different set of 1000 randomly selected patients. The percent of {\it predictive} deviance is plotted as a function of rank in Figure \ref{fig:MedDataPred}. Also plotted on this figure is the percent of predictive deviance by standard PCA. We did not perform LSVD on this data since we were trying to recreate a scenario in which one wouldn't have the ability to solve for principal component scores on the new data. Just as in the explained deviance, the loadings learned from LPCA are able to reconstruct the new data much better than those learned from PCA.

\begin{figure}[t]
	\centering
		\includegraphics{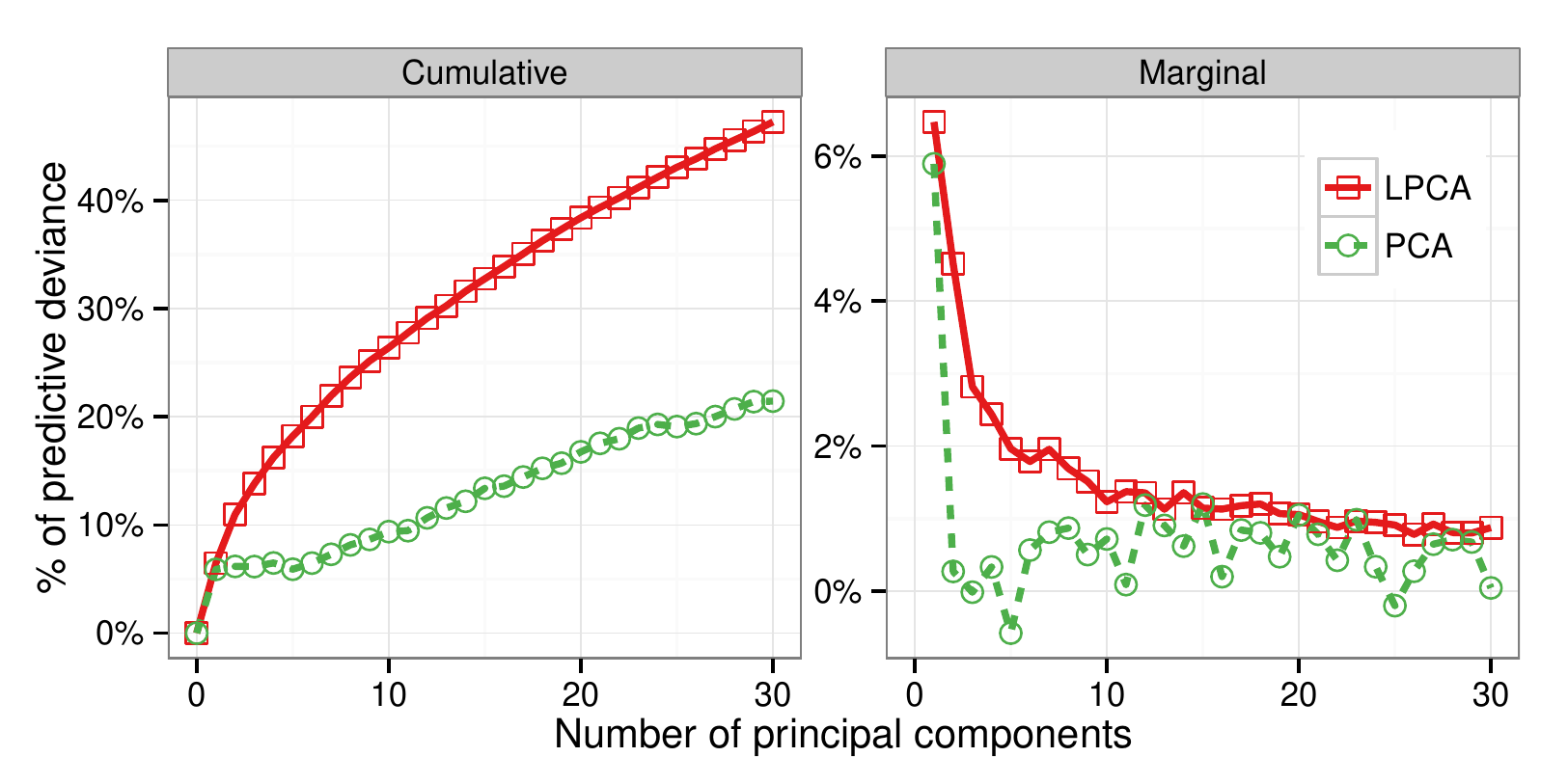}
	\caption{Cumulative and marginal percent of {\it predictive} deviance by the principal components of LPCA and PCA for the patient-diagnosis data.}
	\label{fig:MedDataPred}
\end{figure}

\subsubsection{Evidence of LSVD overfitting through principal component regression}
\label{sec:overfit}
\begin{figure}[t]
	\centering
		\includegraphics[scale=0.75]{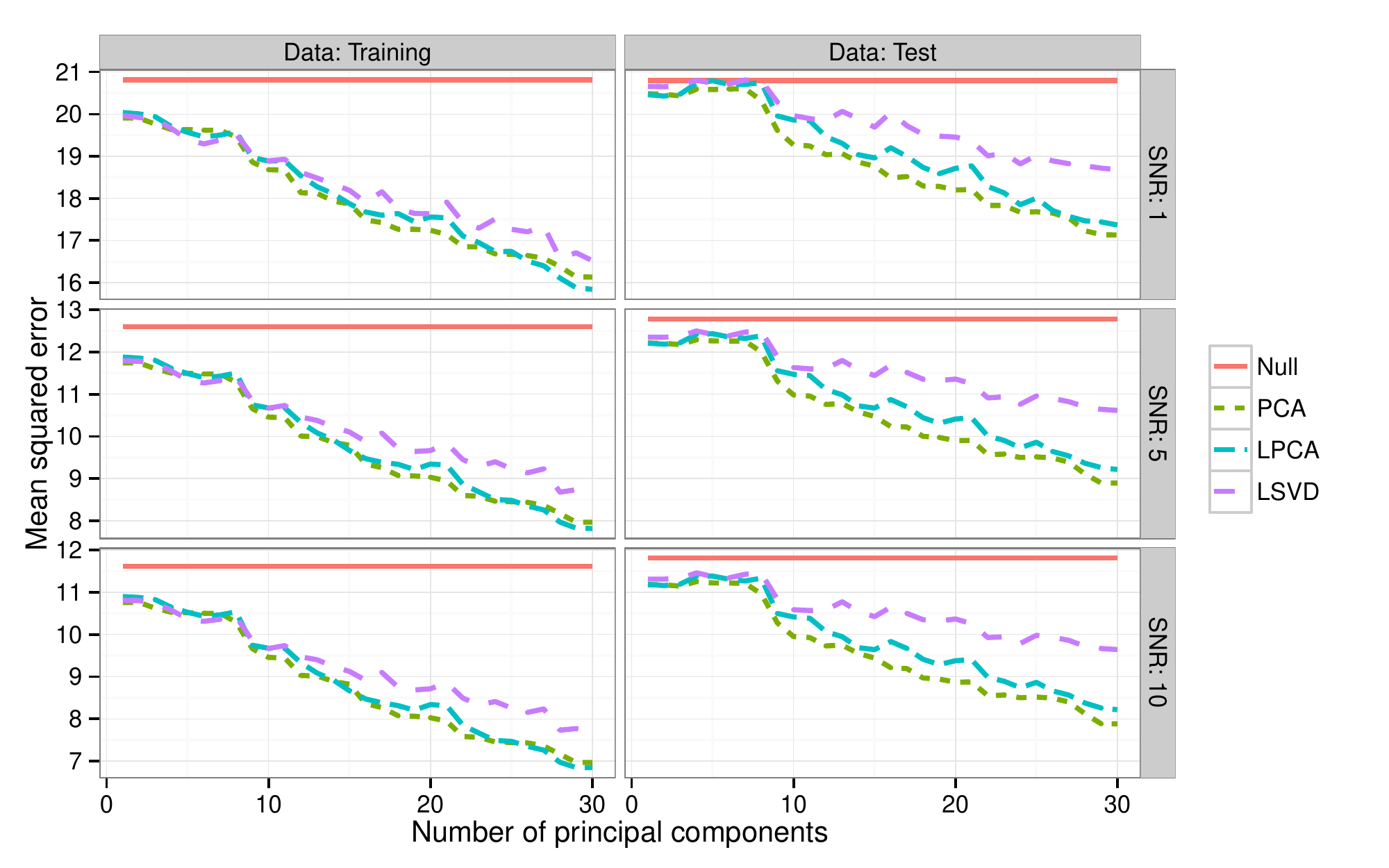}
	\caption{The in-sample and out-of-sample MSE of models for principal component regression based on the patient-diagnosis data}
	\label{fig:MedDataPCR}
\end{figure}

To explore how LSVD might overfit, we created a simulation experiment using the same ICU data. We randomly generated a $d$-dimensional coefficient parameter, $\boldsymbol{\beta}$, from a standard normal distribution. For each patient, we simulated a response variable, $y_i$, from a normal distribution with mean equal to $\bx_i^T \boldsymbol{\beta}$ and the variance, $\sigma^2$. Three different variances were chosen in order for the signal-to-noise (SNR) ratio, $\text{var}(\bx_i^T \boldsymbol{\beta}) / \sigma^2$, to be either 1, 5, or 10. These represent weak, moderate, and strong signals, respectively.
%

For each $k$, we used the $k$ principal components learned from PCA, LPCA, or LSVD as covariates in a linear regression. On a different set of 1000 randomly selected patients, we first derived the principal component scores using the appropriate procedure for each method and then predicted the response for the new patients. As stated in Section \ref{sec:3b:NewData}, obtaining principal component scores on new data with LSVD requires fitting a logistic regression for each new observation, while LPCA and PCA only require matrix multiplications.

The left column in Figure \ref{fig:MedDataPCR} shows the in-sample mean squared error (MSE) and the right column shows the out-of-sample MSE. The three rows of Figure \ref{fig:MedDataPCR} correspond to the three signal-to-noise ratios. For reference, we have also shown the intercept only model (labeled Null). For all three dimensionality reduction methods, the in-sample MSE decreases as the number of components increase, as we would expect. In general, PCA and LPCA track fairly close together. LSVD, on the other hand does markedly worse on both the training and test data, for all three signal levels. Further, the difference between the MSE from LPCA and LSVD is larger on the test data than on the training data. This simulation was run several times with similar results.

The point of this simulation is to show that, even though LSVD may explain more deviance than the other methods due to the extra parameters, that may also cause it to learn loadings that do not generalize well to other settings. This is very important for dimensionality reduction techniques, since they are not usually done in isolation, but rather are part of a larger data analysis. There are no guarantees that LPCA will generalize more effectively than LSVD in every setting, however we have shown in a controlled setting that LSVD produced loadings that do not generalize as well.

\subsubsection{Interpretation of loadings}
\label{sec:4data:Loadings}

\begin{figure}[t]
	\centering
		\includegraphics{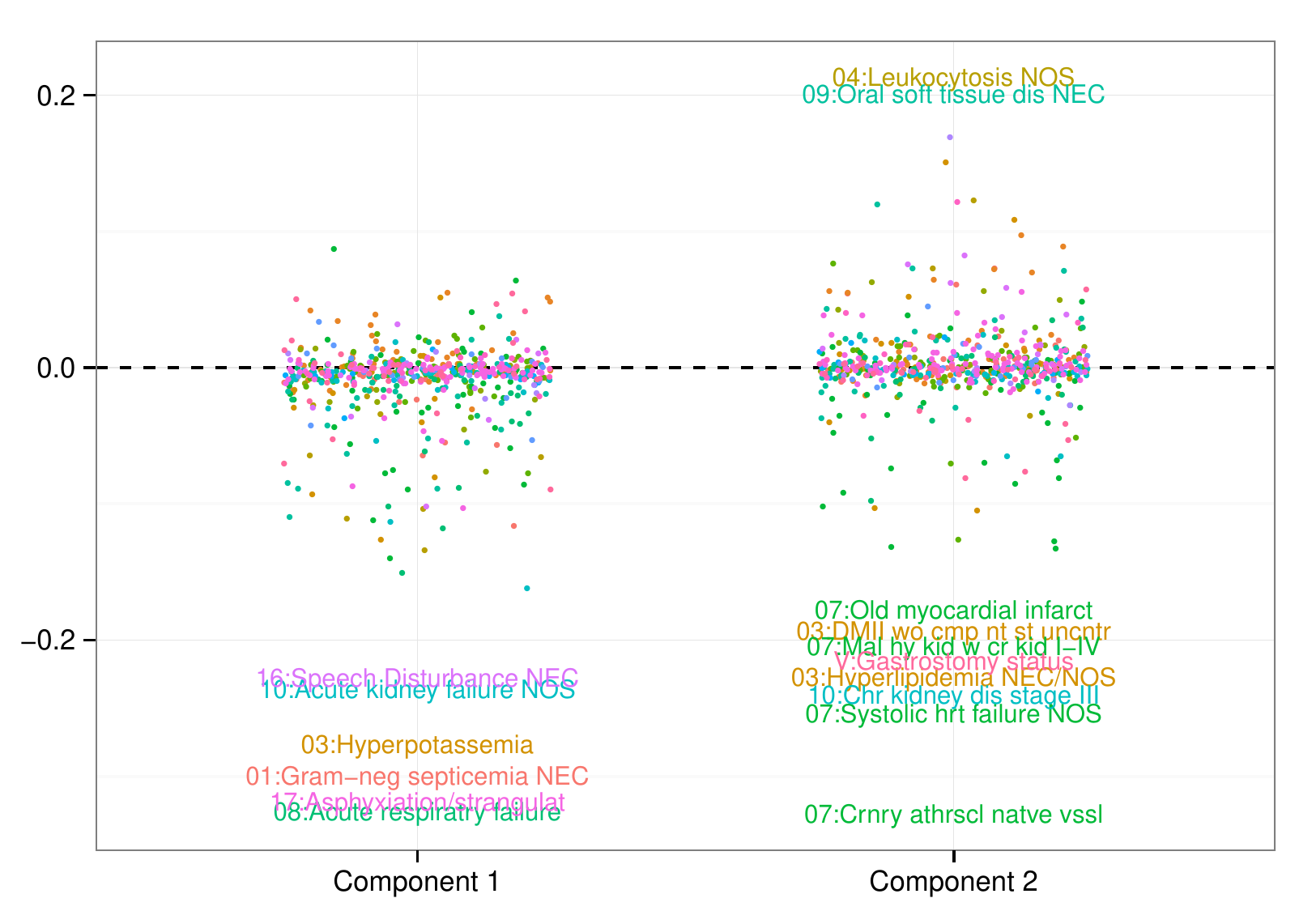}
	\caption{The loadings of two components from logistic PCA of the patient-diagnosis data}
	\label{fig:MedDataInterpret}
\end{figure}

We look at the principal component loadings of LPCA with two components, as chosen by the scree plot, to attempt to interpret the comorbidity in patients. While there is no guarantee that the principal components will be interpretable, insight can occasionally be gained from the analysis. As mentioned in Section \ref{sec:3d:PCs}, the two components do not necessarily have an order as in standard PCA, since they are solved for jointly. Also, it is not guaranteed that either of the two components analyzed here are represented in the analysis with one or three components. 

The resulting loadings are plotted in Figure \ref{fig:MedDataInterpret}. 
We have highlighted the disease categories for the loadings with the highest absolute values. The points are also color-coded according to the broader category that the disease belongs to, which is also related to the number before the label. 

A version of this plot was shown to subject-matter experts in the department of biomedical informatics at Ohio State University and they were able to identify meaning in the components \citep{hyun2013meeting}. The first component has high loadings for acute kidney failure and acute respiratory failure, among others, which are common serious conditions that cause a patient to be admitted into the ICU. 
Three of the large loadings from the second component are diseases of the circulatory system (green): systolic heart failure, myocardial infarction, aneurysm of coronary vessels, and hypertensive renal disease. It also includes two endocrine, nutritional and metabolic diseases, and immunity disorders (brown). The subject-matter experts stated that the diseases with high loadings were ones that they have observed co-occurring relatively often. Based on these findings, it seems that the principal components have a meaningful interpretation related to the diseases patients have at the ICU. 

\section{Discussion}
\label{sec:discussion}

Previous formulations of logistic PCA have extended the singular value decomposition to binary data. Our formulation more consistently extends PCA to binary data by finding projections of the natural parameters of the saturated model. Our method produces principal components which are linear combinations of the data and can be quickly calculated on new data. We have given two algorithms for minimizing the deviance and shown how they perform on both simulated and real data. 

Further, the formulation proposed in this paper can be extended to other members of the exponential family. Using the appropriate deviance and natural parameters from the saturated model, the formulation can naturally be applied to many types of data. 

When $d \gg n$, standard PCA can be inconsistent and it has been shown that adding sparsity constraints to PCA can induce consistency \citep{johnstone2009spca}. Sparse loadings have the additional benefit of easier interpretation. \cite{lee2010sparse} extended LSVD by adding an $L_1$ penalty to the loadings matrix $\bB$. Our formulation can be extended in the same way, but further research is needed to find the best way to solve for the loadings. However, it would be straightforward to incorporate the sparsity constraint with the Fantope solution, similar to how it was done in \cite{vu2013fantope} with standard PCA.

\section*{Acknowledgements}
We would like to thank Vince Vu for his feedback, especially on the convex formulation and algorithm.
We also thank Sookyung Hyun and Cheryl Newton at College of Nursing for providing the medical diagnoses data 
and valuable comments on a preliminary data analysis.
This research was supported in part by National Science Foundation grants DMS-12-09194 and DMS-15-13566.

\appendix
\section{Appendix}
\subsection{Calculation of gradient for logistic PCA}
\label{app:deriv_lpca}
The gradient of the deviance in \eqref{eq:lpca} with respect to $\bU$ can be seen from the steps below. 
\eq{
\frac{1}{2} \frac{\partial D}{\partial \bU} = & - \frac{\partial}{\partial \bU} tr\left( \bX^T \left( \bOne_n \bmu^T + (\Satnat - \bOne_n \bmu^T) \bU \bU^T \right) \right) \\ 
& + \frac{\partial}{\partial \bU} \sum_{i,j} \log \left( 1 + \exp(\mu_j + [\bU\bU^T (\bsatnat_i - \bmu)]_{j}) \right)
}

By standard matrix derivative rules (see, for example, \cite{peterson2012matrix}), 
\begin{equation*}
\frac{\partial}{\partial \bU} tr(\bX^T (\Satnat - \bOne_n \bmu^T) \bU\bU^T) = \left(\bX^T (\Satnat - \bOne_n \bmu^T) + (\Satnat - \bOne_n \bmu^T)^T \bX \right) \bU.
\end{equation*}

To take the derivative of the second piece, letting $\hat{\theta}_{ij} = \mu_j  + [\bU\bU^T (\bsatnat_i - \bmu)]_{j}$, note that 
\begin{eqnarray*}
\frac{\partial}{\partial u_{kl}} \sum_{i,j} \log \left( 1 + \exp(\hat{\theta}_{ij}) \right) 
= \sum_{i,j} \frac{\exp(\hat{\theta}_{ij})}{1 + \exp(\hat{\theta}_{ij})} \frac{\partial \hat{\theta}_{ij}}{\partial u_{kl}} 
= \sum_{i,j} \hat{p}_{ij} \frac{\partial \hat{\theta}_{ij}}{\partial u_{kl}}.
\end{eqnarray*}

Since
\begin{eqnarray*}
\frac{\partial [\bU\bU^T (\bsatnat_i - \bmu)]_{j}}{\partial u_{kl}} &=& 
  \begin{cases}
   (\satnat_{ik} - \mu_k) u_{jl}  & \text{if } k \neq j \\
   (\satnat_{ik} - \mu_k) u_{jl} + (\bsatnat_i - \bmu)^T U_{l}  & \text{if } k=j,
  \end{cases}
\end{eqnarray*}
that makes
\begin{eqnarray*}
\frac{\partial}{\partial u_{kl}} \sum_{i,j} \log \left( 1 + \exp(\hat{\theta}_{ij}) \right) &=& \sum_{i,j} \hat{p}_{ij} (\satnat_{ik} - \mu_k) u_{jl} + \sum_{i} \hat{p}_{ik} (\bsatnat_i - \bmu)^T U_{l} \\
&=& (\Satnatcol_{k} - \bOne_n \mu_k)^T \hat{\bP} U_{l} + \hat{P}_{k}^T (\Satnat - \bOne_n \bmu^T) U_{l}.
\end{eqnarray*}

In matrix notation,
\begin{equation*}
\frac{\partial}{\partial \bU} \sum_{i,j} \log \left( 1 + \exp(\mu_j  + [\bU\bU^T (\bsatnat_i - \bmu)]_{j}) \right) = \left(\hat{\bP}^T (\Satnat - \bOne_n \bmu^T) + (\Satnat - \bOne_n \bmu^T)^T \hat{\bP} \right) \bU,
\end{equation*}
and the result in Equation (\ref{eq:lpca_deriv}) follows.

The gradient of the deviance with respect to $\bmu$ in \eqref{eq:lpca_deriv_mu} is derived as follows.
\eq{
\frac{1}{2} \frac{\partial D}{\partial \bmu} = & - \frac{\partial}{\partial \bmu} tr\left( \bX^T \bOne_n \bmu^T \left( \bI_d - \bU \bU^T \right) \right) \\ 
& + \frac{\partial}{\partial \bmu} \sum_{i,j} \log \left( 1 + \exp(\mu_j  + \left[\bU \bU^T (\bsatnat_i - \bmu) \right]_{j}) \right)
}

Using standard vector differentiation,
\eq{
\frac{\partial}{\partial \bmu} tr\left( \bX^T \bOne_n \bmu^T \left( \bI_d - \bU \bU^T \right) \right) = (\bI_d - \bU\bU^T) \bX^T \bOne_n
}
and
\eq{
\frac{\partial}{\partial \bmu} \sum_{i,j} \log \left( 1 + \exp(\mu_j  + \left[\bU \bU^T (\bsatnat_i - \bmu) \right]_{j}) \right) = \sum_{i,j} \hat{p}_{ij} \left( \be_j - \bu_{j \cdot}^T \bU^T \right) = (\bI_d - \bU\bU^T) \hat{\bP}^T \bOne_n,
}
where $\be_j$ is a length $d$ standard basis vector with $1$ in the $j$th position.

\subsection{Proof of Theorems}

\subsubsection{Theorem \ref{thm:ind}} \label{sec:ind_proof}
\begin{proof}
\item[(i)] When $\bu = \be_l$, the solutions for the main effects $\bmu$ are known analytically. For $j \neq l$, $\hat{\mu}_j = \logit \bar{X}_j$. $\hat{\mu}_l$ is undefined because $\mu_l$ has no effect on the deviance. We will let $\hat{\mu}_l = \logit \bar{X}_l$ for simplicity, although any constant would work. In this case, the estimated natural parameters are
\eq{
\hat{\theta}_{ij} = \hat{\mu}_j + \delta_{jl} (\bsatnat_i - \hat{\bmu})^T \be_l =
  \begin{cases}
   \hat{\mu}_j & \text{if } j \neq l \\
   \satnat_{il}  & \text{if } j = l
  \end{cases},
}
where $\delta_{jl}$ is the Kronecker delta. 

Since $\be_l$ is a standard basis vector, 
\eq{
[\bC^m \be_l]_j = c^m_{jl},
}
where 
\begin{equation} \label{eq:stat_mat}
c^m_{jl} = (X_j - \hat{P}_j)^T (\Satnatcol_l - \bOne_n \hat{\mu}_l) + (X_l - \hat{P}_j)^T (\Satnatcol_j - \bOne_n \hat{\mu}_j).
\end{equation} 

We will show that $c^m_{jl}$ is equal to 0 when $j \neq l$ and $\bar{X}_l = \frac{1}{2}$. Looking at the first part of the summation in \eqref{eq:stat_mat},
\eq{
(X_j - \hat{P}_j)^T (\Satnatcol_l - \bOne_n \hat{\mu}_l) & = (X_j - \bOne_n \sigma(\hat{\mu}_j))^T (m Q_l - \bOne_n \hat{\mu}_l) \\
& = (X_j - \bOne_n \bar{X}_j)^T (m(2 X_l - \bOne_n) - \bOne_n \hat{\mu}_l) \\
& = m \left[ 2 X_j^T X_l - n \bar{X}_j  - n \bar{X}_j \hat{\mu}_l / m - 2 n \bar{X}_j \bar{X}_l + n \bar{X}_j + n \bar{X}_j \hat{\mu}_l / m  \right] \\
& = m \left[ 2 X_j^T X_l - 2 n \bar{X}_j \bar{X}_l \right] \\
& = m \left[ 2 n \bar{X}_j \bar{X}_l - 2 n \bar{X}_j \bar{X}_l  \right] = 0.
}
The last line is due to the assumption that $X_j$ and $X_l$ are uncorrelated.

From the fact that $(x_{il} - \sigma(m q_{il}))(m q_{ij} - \hat{\mu}_j) = (m q_{ij} q_{il} - q_{il} \hat{\mu}_j)/(1 + e^m)$, the second part of $c^m_{jl}$ in \eqref{eq:stat_mat} is
\eq{
(X_l - \hat{P}_l)^T (\Satnatcol_j - \bOne_n \hat{\mu}_j) & = (X_l - \sigma(m Q_l))^T (m Q_j- \bOne_n \hat{\mu}_j) \\
& = \sum_{i = 1}^n (x_{il} - \sigma(m q_{il})) (m q_{ij}- \hat{\mu}_j) \\
& = \sum_{i = 1}^n \frac{m q_{ij} q_{il} - q_{il} \hat{\mu}_j}{1 + e^m} \\
& = n \bar{Q}_{l} \frac{m \bar{Q}_{j} - \hat{\mu}_j}{1 + e^m}.
}
If $\bar{Q}_{l} = 0$, or equivalently $\bar{X}_l = \frac{1}{2}$, then this is exactly zero, for all $m$.

When $j = l$, also using the fact that $q_{il}^2 = 1$,
\eq{
c^m_{ll} & = 2 (X_l - \hat{P}_l)^T (\Satnatcol_l - \bOne_n \hat{\mu}_l) \\
& = 2 \sum_{i = 1}^n \frac{m q_{il}^2 - q_{il} \hat{\mu}_l}{1 + e^m} \\
& = 2 n \frac{m - \hat{\mu}_l \bar{Q}_{l}}{1 + e^m}.
}

When $\bar{X}_l = \frac{1}{2}$, 
\eq{
\bC^m \be_l = \lambda_m \be_l,
}
for all $m$, where $\lambda_m = \frac{2 n m}{1 + e^m}$. With $\bmu = \hat{\bmu}$ and $\bu = \be_l$, the first-order optimality conditions (\ref{eq:lpca_deriv}--\ref{eq:ortho}) are exactly satisfied.

\item[(ii)]
When $\bar{X}_l \neq \frac{1}{2}$, the first part of \eqref{eq:stat_mat} still equals zero, but the other part does not. The squared norm of equation \eqref{eq:stat_mat} equals
\eq{
\| \bC^m \be_l \|^2 = \left( 2n \frac{m - \hat{\mu}_l \bar{Q}_{l}}{1 + e^m} \right)^2 + \sum_{j : j \neq l} \left(n \bar{Q}_{l} \frac{m \bar{Q}_{j} - \hat{\mu}_j}{1 + e^m}\right)^2,
}
which can be made as small as we desire by increasing $m$.
\end{proof}

\subsubsection{Theorem \ref{thm:order}} \label{sec:order_proof}
\begin{proof}
If $\bu = \be_l$ and $\hat{\mu}_j = \logit \bar{X}_j$ for $j \neq l$, then the deviance of the $l$th column does not depend on the column mean $\bar{X}_l$ and is given by
\eq{
- 2 \sum_{i = 1}^n \log \sigma(q_{il} \satnat_{il}) = - 2 n \log \sigma(m).
}
Thus, the deviance depends on the other $d - 1$ columns of the dataset. The deviance of the $j$th column ($j \neq l$) is
\eq{
- 2 n \left( \bar{X}_j \log \bar{X}_j + (1 - \bar{X}_j) \log (1 - \bar{X}_j) \right),
}
which is maximized at $\bar{X}_j = \frac{1}{2}$ and decreases as $\bar{X}_j$ moves away from $\frac{1}{2}$. Therefore, choosing $l$ with the column mean closest to $\frac{1}{2}$ will result in a fit with the lowest deviance.
\end{proof}

\subsubsection{Theorem \ref{thm:compound}} \label{sec:cs_proof}
\begin{proof}
The first-order optimality condition for the loading vector with $k = 1$ and no main effects is
\eq{
\bC^m \bu = \lambda_m \bu,
}
where 
\eq{
\bC^m := (\bX - \hat{\bP})^T \Satnat + \Satnat^T (\bX - \hat{\bP}).
}

We will show that, with $\bu = \frac{1}{\sqrt{d}} \bOne_d$ and the conditions listed in the theorem, $\bC^m$ is compound symmetric, which in turn implies that $\bu$ satisfies the first-order optimality conditions. 

One useful implication of $\bu \propto \bOne_d$ is that $\hat{p}_{ij} = \hat{p}_{ik}$ for all $i$ and $j, k$ because $\hat{\theta}_{ij} = u_j (\bu^T \bsatnat_i) = \frac{m}{d} \sum_{l = 1}^d q_{il}$. We will therefore let $\hat{\mathbf{p}}$ be the column vector with $i$th element equal to $\hat{p}_{ij}$ for all $j$.

First, we show that $c^m_{jj} = c^m_{kk},$ for all $j, k$. $c^m_{jj} = c^m_{kk}$ if and only if $X_j^TQ_j - \hat{P}_j^TQ_j = X_k^TQ_k - \hat{P}_k^TQ_k$. This, in turn, is equivalent to 
\eq{
\hat{\mathbf{p}}^T (Q_j - Q_k) = \frac{1}{2} \bOne_n^T (Q_j - Q_k).
}

Focusing on the left hand side,
\eq{
\hat{\mathbf{p}}^T (Q_j - Q_k) & = \sum_{i = 1}^n \hat{p}_i (q_{ij} - q_{ik}) \\
& = \sum_{i: q_{ij} \neq q_{ik}} \hat{p}_i (q_{ij} - q_{ik}) 
}
The second line is due to the summation only being non-zero when $q_{ij} \neq q_{ik}$. When this is true, $\sum_{l = 1}^d q_{il} = \sum_{l \not \in \{j,k\}} q_{il}$. If $q_{ij} = q_{ik}$, $\hat{p}_i (q_{ij} - q_{ik})$ will equal 0 regardless of $\hat{p}_i$. Therefore, we can state 
\eq{
\sum_{i: q_{ij} \neq q_{ik}} \hat{p}_i (q_{ij} - q_{ik}) = \sum_{i = 1}^n \sigma \left( \frac{m}{d} \sum_{l \not \in \{j,k\}} q_{il} \right) (q_{ij} - q_{ik}),
}
and from \eqref{eq:cs_cond},
\eq{
\hat{\mathbf{p}}^T (Q_j - Q_k) & = \left( \frac{1}{2} \bOne_n + \sum_{l \not \in \{j,k\}} Q_{l} \beta_{jk,l} \right)^T (Q_j - Q_k) \\
& = \frac{1}{2} \bOne_n^T (Q_j - Q_k).
}
The last line is due to the compound symmetry of $\bQ^T \bQ$, since all the off-diagonal elements are equal to each other. This proves that $c^m_{jj} = c^m_{kk}$.

We will now show that $c^m_{jk} = c^m_{lr}$, as long as $j \neq k$ and $l \neq r$. An off-diagonal element of $\bC^m$ is
\eq{
c^m_{jk} & = m (X_j - \hat{\mathbf{p}})^T Q_k + m (X_k - \hat{\mathbf{p}})^T Q_j \\
& = m \left[ Q_j^TQ_k + \frac{1}{2} \bOne_n^T (Q_j + Q_k) - \hat{\mathbf{p}}^T(Q_j + Q_k) \right].
}
Showing $c^m_{jk} = c^m_{lr}$ is equivalent to showing
\eq{
Q_j^TQ_k + \frac{1}{2} \bOne_n^T (Q_j + Q_k) - \hat{\mathbf{p}}^T(Q_j + Q_k) & =
 Q_l^TQ_r + \frac{1}{2} \bOne_n^T (Q_l + Q_r) - \hat{\mathbf{p}}^T(Q_l + Q_r) \\
\frac{1}{2} \bOne_n^T (Q_j + Q_k) - \hat{\mathbf{p}}^T(Q_j + Q_k) & =
 \frac{1}{2} \bOne_n^T (Q_l + Q_r) -  \hat{\mathbf{p}}^T(Q_l + Q_r),
}
where the terms cancel because $Q_j^TQ_k = Q_l^TQ_r$. Rearranging the terms,
\eq{
\hat{\mathbf{p}}^T(Q_j - Q_r) - \frac{1}{2} \bOne_n^T (Q_j - Q_r) & =
 \hat{\mathbf{p}}^T(Q_l - Q_k) - \frac{1}{2} \bOne_n^T (Q_l - Q_k),
}
where we can see that both sides equal 0 because, as we have already proven, $c^m_{jj} = c^m_{rr}$ and $c^m_{ll} = c^m_{kk}$.

Therefore, $\bC^m$ is compound symmetric for all $m$. This implies that $\bu = \frac{1}{\sqrt{d}} \bOne_d$ is an eigenvector of $\bC^m$ and (in conjunction with the fact that $\bu^T\bu = 1$) $\bu$ satisfies the first-order optimality conditions (\ref{eq:lpca_deriv}, \ref{eq:ortho}).
\end{proof}

\subsubsection{Theorem \ref{thm:lipschitz}} \label{sec:lipschitz_proof}

Next, we show that the gradient is Lipschitz continuous, for which we need the following Lemma.

\begin{lemma}
\label{lem:shrink}
For all $\theta_1, \theta_2 \in \mathbb{R}$,
\eq{
4 | \sigma(\theta_2) - \sigma(\theta_1) | \leq | \theta_2 - \theta_1 |.
}
\end{lemma}
\begin{proof}[Proof of Lemma \ref{lem:shrink}]
Without loss of generality, assume that $\theta_2 \geq \theta_1$. Let $\Delta = \theta_2 - \theta_1 \geq 0$. We want to show that 
\eq{
4 ( \sigma(\theta_1 + \Delta) - \sigma(\theta_1) ) \leq \Delta.
}
The left and right hand sides both equal 0 when $\Delta$ is 0. The derivative with respect to $\Delta$ of the right hand side is 1. The derivative with respect to $\Delta$ of the left hand side is
\eq{
4 \sigma(\theta_1 + \Delta) (1 - \sigma(\theta_1 + \Delta) ) \leq 4 \frac{1}{4} = 1.
}
Since both sides of the inequality are equal to each other at the origin ($\Delta = 0$) and the left hand side increases at a slower rate than the right hand side for $\Delta > 0$, the left hand side is always less than or equal to the right hand side for $\Delta \geq 0$.
\end{proof}

\begin{proof}[Proof of Theorem \ref{thm:lipschitz}]
To prove the Lipschitz continuity, we will work with the squared norm. Let $\hat{\bP}_1$ and $\hat{\bP}_2$ be the matrices of fitted probabilities using $\bH_1$ and $\bH_2$ respectively. First, for the un-symmetrized gradient, letting $\Satnat_c := \Satnat - \bOne_n \bmu^T$, see that
\eq{
\|2 (\hat{\bP}_1 - \bX)^T \Satnat_c - 2 (\hat{\bP}_2 - \bX)^T \Satnat_c \|^2_F & = 4 \|(\hat{\bP}_1^T - \hat{\bP}_2)^T \Satnat_c \|^2_F \\
& \leq 4 \| \Satnat_c \|^2_F \|\hat{\bP}_1 - \hat{\bP}_2 \|^2_F \\
& \leq \frac{1}{4} \| \Satnat_c \|^2_F \| \Satnat_c (\bH_1 - \bH_2) \|^2_F \\
& \leq \left( \frac{\| \Satnat_c \|^2_F}{2} \right)^2 \| \bH_1 - \bH_2 \|^2_F.
}
The first and third inequalities are due to the Cauchy-Schwarz inequality and the second one is due to Lemma \ref{lem:shrink}.

To prove Lipschitz continuity for the symmetrized gradient (removing the $\bX^T \Satnat_c$s which cancel out), check that
\eq{
& 4 \| \hat{\bP}_1^T \Satnat_c + \Satnat_c^T \hat{\bP}_1 - \text{diag}(\hat{\bP}_1^T \Satnat_c) - \hat{\bP}_2^T \Satnat_c - \Satnat_c^T \hat{\bP}_2 + \text{diag}(\hat{\bP}_2^T \Satnat_c) \|_F^2 \\
& ~~ \leq 4 \| \hat{\bP}_1^T \Satnat_c + \Satnat_c^T \hat{\bP}_1 - \hat{\bP}_2^T \Satnat_c - \Satnat_c^T \hat{\bP}_2 \|_F^2 \\
& ~~ = 8 \| (\hat{\bP}_1 - \hat{\bP}_2)^T \Satnat_c \|_F^2 + 
8  \langle (\hat{\bP}_1 - \hat{\bP}_2)^T \Satnat_c, \Satnat_c^T (\hat{\bP}_1 - \hat{\bP}_2) \rangle  \\
& ~~ \leq 8 \| (\hat{\bP}_1 - \hat{\bP}_2)^T \Satnat_c \|_F^2 + 
8 \| (\hat{\bP}_1 - \hat{\bP}_2)^T \Satnat_c \|_F \| \Satnat_c^T (\hat{\bP}_1 - \hat{\bP}_2) \|_F \\
& ~~ = 16 \| (\hat{\bP}_1 - \hat{\bP}_2)^T \Satnat_c \|_F^2 \\
& ~~ \leq 4 \left( \frac{\|\Satnat_c \|_F^2}{2} \right)^2 \| \bH_1 - \bH_2 \|^2_F.
}
The first inequality is true because the difference on the diagonal is doubled without the diagonal correction, the second is due to the Cauchy-Schwarz inequality, and the last is from the previous proof for the un-symmetrized gradient. Therefore, the symmetrized gradient is Lipschitz continuous with the bound stated above.
\end{proof}

\subsection{Further simulation results}
\subsubsection{Number of iterations to convergence}
\label{app:iters}

\begin{figure}[t]
	\centering
		\includegraphics[scale=0.75]{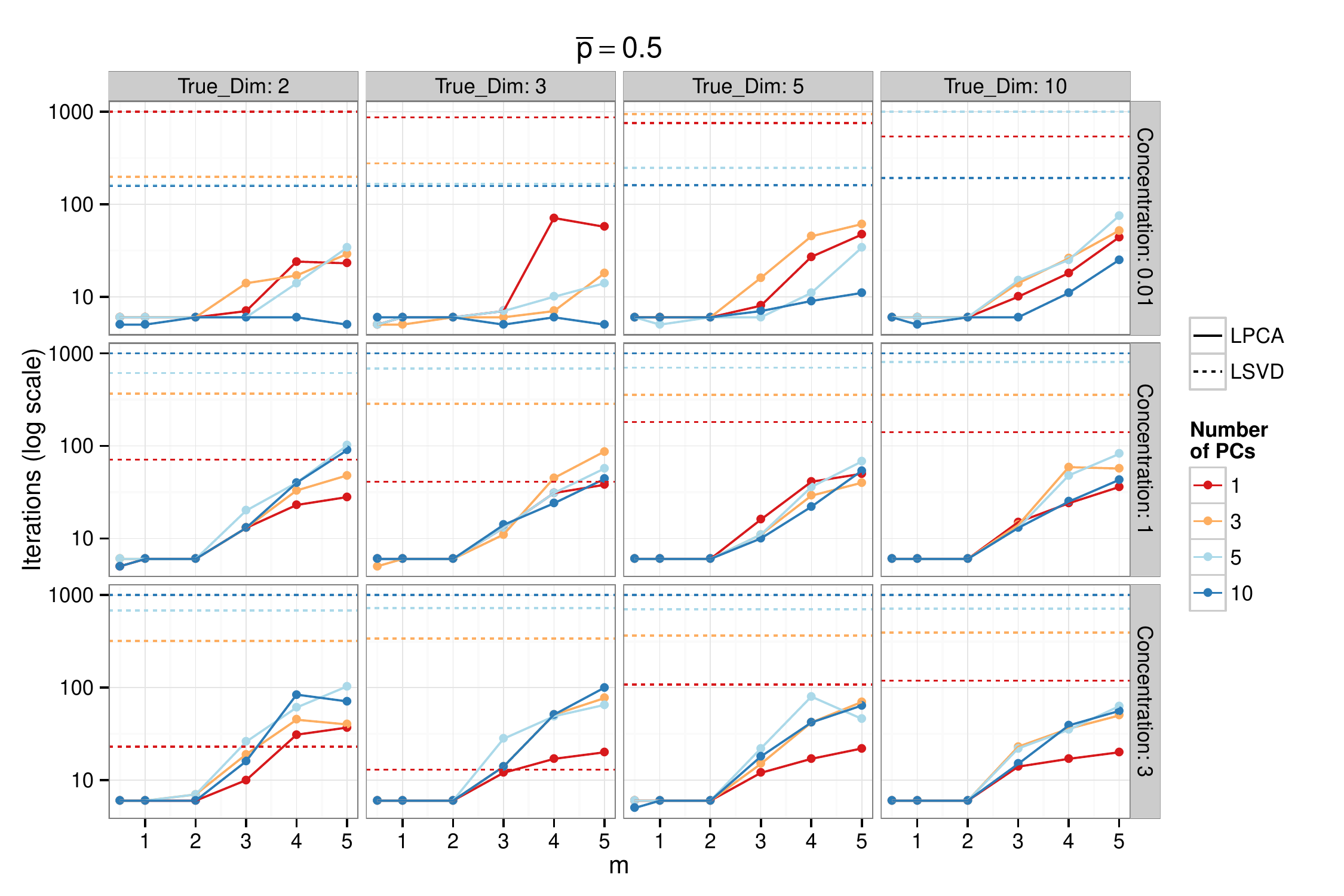}
	\caption{Number of iterations required for convergence of the logistic PCA and logistic SVD algorithms for a variety of scenarios}
	\label{fig:Iters05}
\end{figure}

For the same simulation setup as in Figure \ref{fig:PSqerror05}, we kept track of the number of iterations required until convergence for both the LSVD algorithm and the LPCA algorithm. The initial value of $\bU$ or $\bB$ was the first $k$ right singular vectors of $\bQ$ for both. Each iteration for the two algorithms is comparable in that they involve computing the working variables of the majorization function and finding either an SVD or an eigen-decomposition. Both of the algorithms terminate if the difference of the average deviance between successive iterations is less than $10^{-5}$ or if the maximum number of iterations (1000) is reached.

It is apparent from Figure \ref{fig:Iters05} that the number of iterations required for convergence is a lot less for LPCA than for LSVD, in general, and in many cases LSVD reaches the maximum. (Note that the y-axis is on the log scale.) For $\phi \in \{1,3\}$, the number of iterations for LSVD increases as the number of estimated components increases. Also, as $m$ increases, the number of iterations for LPCA increases, in general.

We believe that LSVD requires more iterations to converge because there are more parameters for it to optimize for, so there are many little updates that can be made to the parameters which improve the deviance enough. One explanation for why it takes LPCA with large $m$ longer to converge is that small changes in $\bU$ can have a relatively large effect in the estimated natural parameters with large $m$.

\subsubsection{Simulation results for more sparse situation}
\label{app:sim01}

\begin{figure}[t]
	\centering
		\includegraphics[scale=0.75]{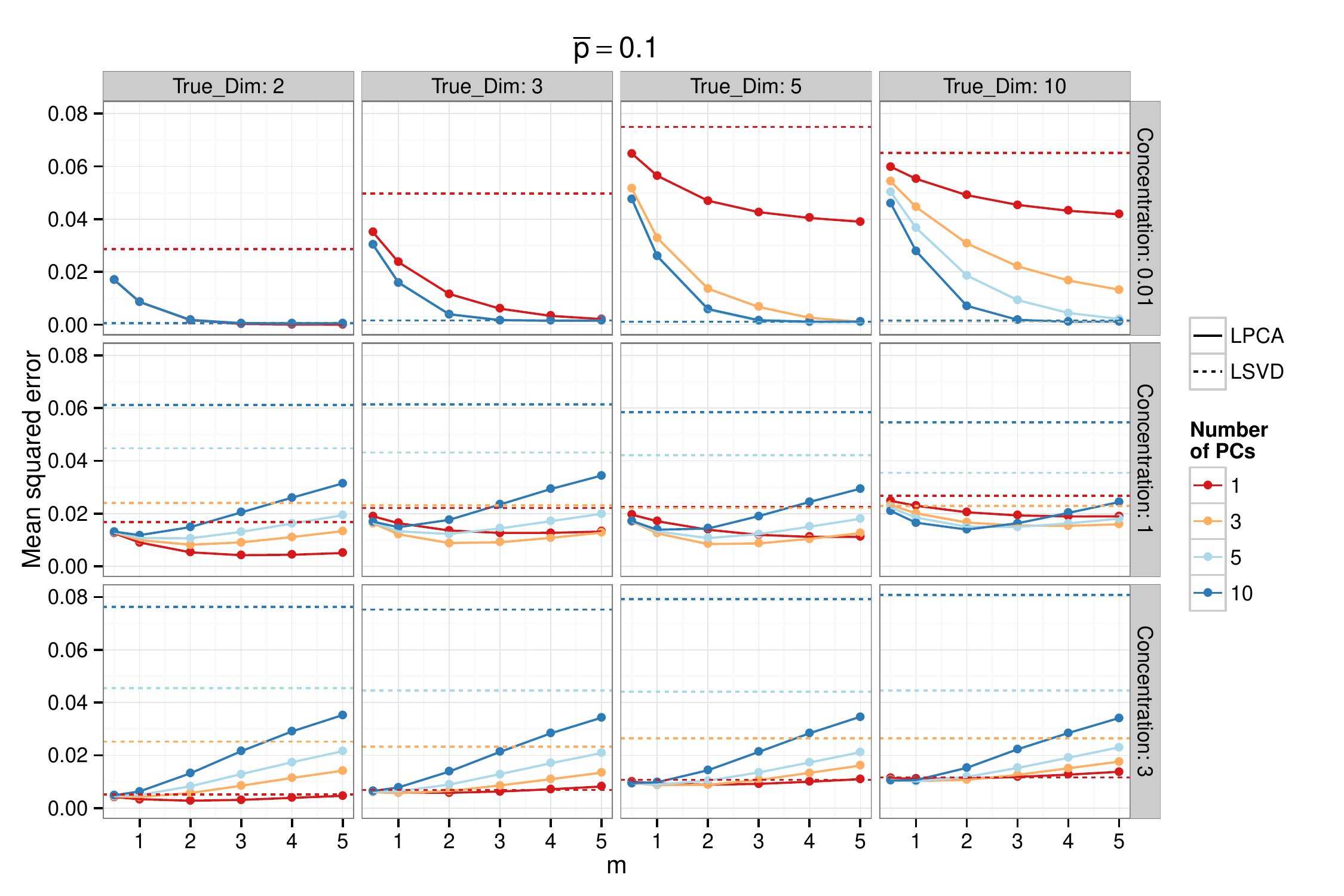}
	\caption{Mean squared error of probability estimates when $\bar{p}=0.1$ for a variety of scenarios}
	\label{fig:PSqerror01}
\end{figure}

Just like in Section \ref{sec:4sim:Probs}, we simulated from a variety of scenarios, only this time $\bar{p}=0.1$, for a situation where the binary matrix consists of mostly 0's. The estimated LSVD and LPCA included main effects. The results in Figure \ref{fig:PSqerror01} are mostly the same as with $\bar{p}=0.5$ except that an estimate of the probability matrix with a rank lower than the true rank does not degrade performance as badly.

\subsubsection{Deviance comparison of LPCA and LSVD}
\label{app:sim:dev}

\begin{figure}[t]
	\centering
		\includegraphics[scale=0.75]{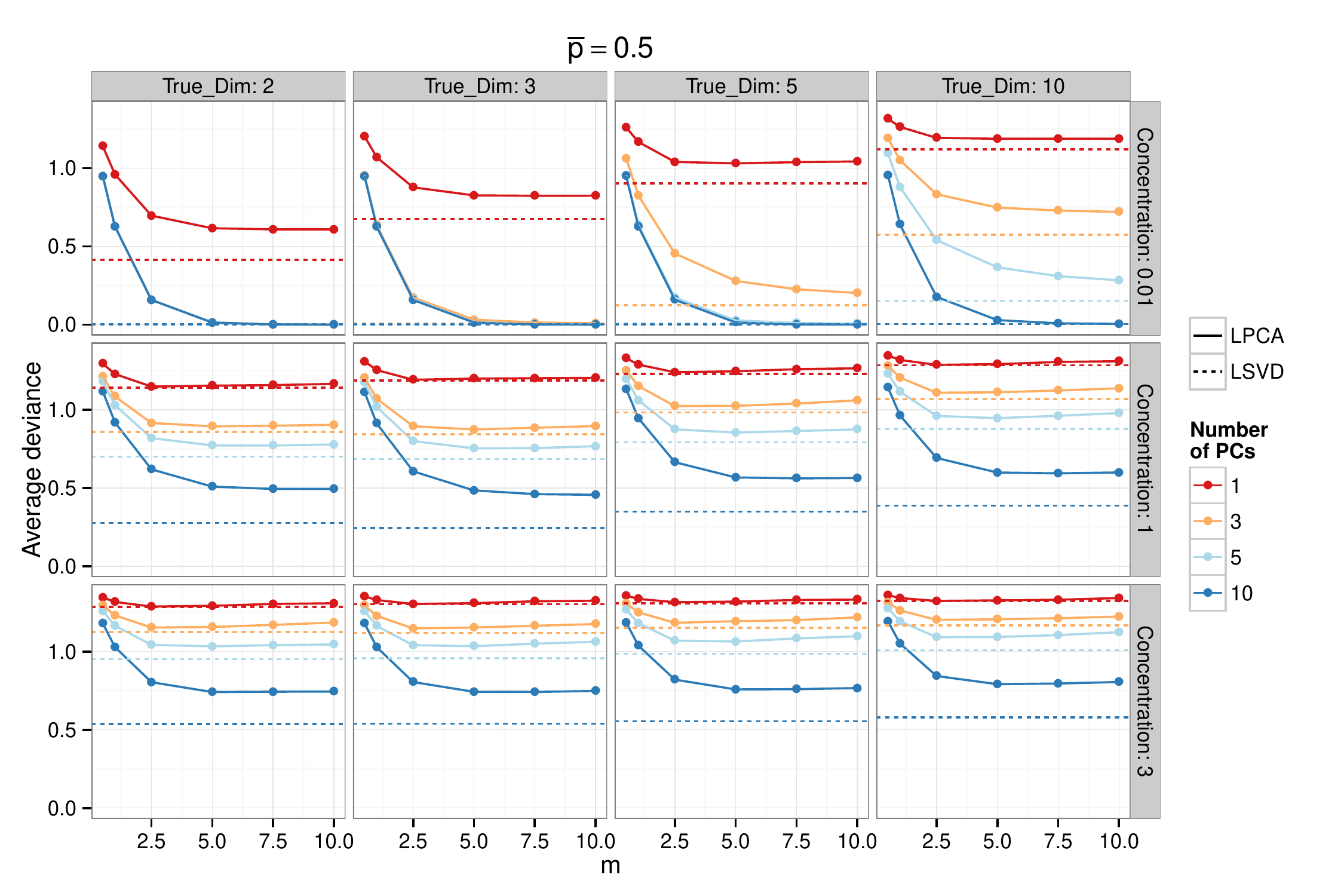}
	\caption{Average deviance of probability estimates when $\bar{p}=0.5$ for a variety of scenarios}
	\label{fig:Deviances05}
\end{figure}

Finally, under the same setup as before with $\bar{p}=0.5$, Figure \ref{fig:Deviances05} shows the deviance of the estimates. As expected, LSVD has lower deviance or equivalent deviance for all scenarios. The improvement from LSVD is smallest when $\phi = 0.01$. When the probabilities are more concentrated, however, LSVD clearly has a lower deviance than LPCA. It also appears that the higher the estimated rank is, the greater the decrease in the deviance is from LSVD. 

\bibliographystyle{chicago}
\bibliography{logistic_pca}

\begin{thebibliography}{}

\bibitem[\protect\citeauthoryear{Beck and Teboulle}{Beck and
  Teboulle}{2009}]{beck2009fista}
Beck, A. and M.~Teboulle (2009).
\newblock A fast iterative shrinkage-thresholding algorithm for linear inverse
  problems.
\newblock {\em SIAM Journal on Imaging Sciences\/}~{\em 2\/}(1), 183--202.

\bibitem[\protect\citeauthoryear{Boyd, Xiao, and Mutapcic}{Boyd
  et~al.}{2003}]{boyd2003subgradient}
Boyd, S., L.~Xiao, and A.~Mutapcic (2003).
\newblock Subgradient methods.
\newblock {\em Lecture notes of EE392o, Stanford University, Autumn
  Quarter\/}~{\em 2004}.

\bibitem[\protect\citeauthoryear{Collins, Dasgupta, and Schapire}{Collins
  et~al.}{2001}]{CollinsEtAl2001}
Collins, M., S.~Dasgupta, and R.~E. Schapire (2001).
\newblock A generalization of principal components analysis to the exponential
  family.
\newblock In T.~Dietterich, S.~Becker, and Z.~Ghahramani (Eds.), {\em Advances
  in Neural Information Processing Systems 14}, pp.\  617--624. MIT Press.

\bibitem[\protect\citeauthoryear{Dattorro}{Dattorro}{2005}]{dattorro2005convex}
Dattorro, J. (2005).
\newblock {\em Convex optimization \& Euclidean distance geometry}.
\newblock Meboo Publishing USA.

\bibitem[\protect\citeauthoryear{de~Leeuw}{de~Leeuw}{2006}]{Deleeuw2006}
de~Leeuw, J. (2006).
\newblock Principal component analysis of binary data by iterated singular
  value decomposition.
\newblock {\em Computational Statistics and Data Analysis\/}~{\em 50\/}(1),
  21--39.

\bibitem[\protect\citeauthoryear{Eckart and Young}{Eckart and
  Young}{1936}]{eckart1936svd}
Eckart, C. and G.~Young (1936).
\newblock The approximation of one matrix by another of lower rank.
\newblock {\em Psychometrika\/}~{\em 1\/}(3), 211--218.

\bibitem[\protect\citeauthoryear{Fan}{Fan}{1949}]{fan1949eig}
Fan, K. (1949).
\newblock On a theorem of {W}eyl concerning eigenvalues of linear
  transformations {I}.
\newblock {\em Proceedings of the National Academy of Sciences of the United
  States of America\/}~{\em 35\/}(11), 652--655.

\bibitem[\protect\citeauthoryear{Hastie, Tibshirani, and Friedman}{Hastie
  et~al.}{2009}]{ESL}
Hastie, T., R.~Tibshirani, and J.~Friedman (2009).
\newblock {\em The Elements of Statistical Learning}, Volume~2.
\newblock Springer.

\bibitem[\protect\citeauthoryear{Hotelling}{Hotelling}{1933}]{hotelling1933pca}
Hotelling, H. (1933).
\newblock Analysis of a complex of statistical variables into principal
  components.
\newblock {\em Journal of Educational Psychology\/}~{\em 24\/}(6), 417--441.

\bibitem[\protect\citeauthoryear{Hunter and Lange}{Hunter and
  Lange}{2004}]{hunter2004MM}
Hunter, D.~R. and K.~Lange (2004).
\newblock A tutorial on {MM} algorithms.
\newblock {\em The American Statistician\/}~{\em 58\/}(1), 30--37.

\bibitem[\protect\citeauthoryear{Hyun and Newton}{Hyun and
  Newton}{2013}]{hyun2013meeting}
Hyun, S. and C.~Newton (2013, December).
\newblock Personal communication.

\bibitem[\protect\citeauthoryear{Johnson}{Johnson}{2014}]{johnson2014logistic}
Johnson, C.~C. (2014).
\newblock Logistic matrix factorization for implicit feedback data.
\newblock In {\em Advances in Neural Information Processing Systems 27:
  Distributed Machine Learning and Matrix Computations Workshop}.

\bibitem[\protect\citeauthoryear{Johnstone and Lu}{Johnstone and
  Lu}{2009}]{johnstone2009spca}
Johnstone, I.~M. and A.~Y. Lu (2009).
\newblock On consistency and sparsity for principal components analysis in high
  dimensions.
\newblock {\em Journal of the American Statistical Association\/}~{\em
  104\/}(486), 682--693.

\bibitem[\protect\citeauthoryear{Jolliffe}{Jolliffe}{2002}]{jolliffe2005principal}
Jolliffe, I. (2002).
\newblock {\em Principal component analysis}.
\newblock Springer.

\bibitem[\protect\citeauthoryear{Lange}{Lange}{2013}]{lange2013optimization}
Lange, K. (2013).
\newblock {\em Optimization}.
\newblock Springer.

\bibitem[\protect\citeauthoryear{Lee, Huang, and Hu}{Lee
  et~al.}{2010}]{lee2010sparse}
Lee, S., J.~Z. Huang, and J.~Hu (2010).
\newblock Sparse logistic principal components analysis for binary data.
\newblock {\em The Annals of Applied Statistics\/}~{\em 4\/}(3), 1579--1601.

\bibitem[\protect\citeauthoryear{Li and Tao}{Li and Tao}{2010}]{li2010simple}
Li, J. and D.~Tao (2010).
\newblock Simple exponential family {PCA}.
\newblock In {\em International Conference on Artificial Intelligence and
  Statistics}, pp.\  453--460.

\bibitem[\protect\citeauthoryear{McCullagh and Nelder}{McCullagh and
  Nelder}{1989}]{mccullagh1989glm}
McCullagh, P. and J.~A. Nelder (1989).
\newblock {\em Generalized linear models ({S}econd edition)}.
\newblock London: Chapman \& Hall.

\bibitem[\protect\citeauthoryear{Nesterov}{Nesterov}{2007}]{nesterov2007gradient}
Nesterov, Y.~E. (2007).
\newblock Gradient methods for minimizing composite objective function.
\newblock Technical report, Center for Operations Research and Econometrics
  (CORE), Catholic University of Louvain.

\bibitem[\protect\citeauthoryear{Pearson}{Pearson}{1901}]{pearson1901pca}
Pearson, K. (1901).
\newblock On lines and planes of closest fit to systems of points in space.
\newblock {\em The London, Edinburgh, and Dublin Philosophical Magazine and
  Journal of Science\/}~{\em 2\/}(11), 559--572.

\bibitem[\protect\citeauthoryear{Petersen and Pedersen}{Petersen and
  Pedersen}{2012}]{peterson2012matrix}
Petersen, K.~B. and M.~S. Pedersen (2012, {N}ovember).
\newblock The matrix cookbook.
\newblock Version 20121115.

\bibitem[\protect\citeauthoryear{{R Core Team}}{{R Core Team}}{2015}]{R2015}
{R Core Team} (2015).
\newblock {\em R: A Language and Environment for Statistical Computing}.
\newblock Vienna, Austria: R Foundation for Statistical Computing.

\bibitem[\protect\citeauthoryear{Schein, Saul, and Ungar}{Schein
  et~al.}{2003}]{ScheinEtAl2003}
Schein, A.~I., L.~K. Saul, and L.~H. Ungar (2003).
\newblock A generalized linear model for principal component analysis of binary
  data.
\newblock In {\em Proceedings of the Ninth International Workshop on Artificial
  Intelligence and Statistics}, Volume~38.

\bibitem[\protect\citeauthoryear{Tipping}{Tipping}{1998}]{Tipping1998}
Tipping, M.~E. (1998).
\newblock Probabilistic visualisation of high-dimensional binary data.
\newblock In M.~Kearns, S.~Solla, and D.~Cohn (Eds.), {\em Advances in Neural
  Information Processing Systems 11}, pp.\  592--598. MIT Press.

\bibitem[\protect\citeauthoryear{Tipping and Bishop}{Tipping and
  Bishop}{1999}]{tipping1999probabilistic}
Tipping, M.~E. and C.~M. Bishop (1999).
\newblock Probabilistic principal component analysis.
\newblock {\em Journal of the Royal Statistical Society: Series B (Statistical
  Methodology)\/}~{\em 61\/}(3), 611--622.

\bibitem[\protect\citeauthoryear{Vu, Cho, Lei, and Rohe}{Vu
  et~al.}{2013}]{vu2013fantope}
Vu, V.~Q., J.~Cho, J.~Lei, and K.~Rohe (2013).
\newblock Fantope projection and selection: A near-optimal convex relaxation of
  sparse pca.
\newblock In C.~Burges, L.~Bottou, M.~Welling, Z.~Ghahramani, and K.~Weinberger
  (Eds.), {\em Advances in Neural Information Processing Systems 26}, pp.\
  2670--2678. Curran Associates, Inc.

\bibitem[\protect\citeauthoryear{Welling, Chemudugunta, and Sutter}{Welling
  et~al.}{2008}]{welling2008deterministic}
Welling, M., C.~Chemudugunta, and N.~Sutter (2008).
\newblock Deterministic latent variable models and their pitfalls.
\newblock In {\em SIAM International Conference on Data Mining}, pp.\
  196--207.

\bibitem[\protect\citeauthoryear{Wen and Yin}{Wen and Yin}{2013}]{wen2013ortho}
Wen, Z. and W.~Yin (2013).
\newblock A feasible method for optimization with orthogonality constraints.
\newblock {\em Mathematical Programming\/}~{\em 142\/}(1-2), 397--434.

\end{thebibliography}

\end{document}